\documentclass[11pt]{article}

\usepackage[margin = 1in]{geometry}
\usepackage[utf8]{inputenc} 
\usepackage[T1]{fontenc}    
\usepackage[hidelinks]{hyperref}       
\usepackage{url}            
\usepackage{booktabs}       
\usepackage{amsfonts}       
\usepackage{tabularx}
\usepackage[ruled,noend]{algorithm2e} 
\usepackage{amsmath,amsthm,amssymb,xr}
\usepackage{thmtools}
\usepackage{thm-restate}
\usepackage{setspace}
\usepackage[numbers, sort&compress]{natbib}
\usepackage{etoolbox}

\providetoggle{singlecolumnequations}
\settoggle{singlecolumnequations}{true}

\SetKwProg{Fn}{Function}{}{}
\usepackage{float,subcaption,placeins,xcolor,enumitem}
\usepackage[]{graphicx}
\usepackage[title]{appendix}

\SetCommentSty{mycommfont}
\SetAlFnt{\small}
\SetAlCapFnt{\small}
\SetAlCapNameFnt{\small}
\SetAlCapHSkip{0pt}
\IncMargin{-\parindent}

\newtheorem{theorem}{Theorem}[section]

\newtheorem{remark}{Remark}[section]
\newtheorem{corollary}{Corollary}[section]
\newtheorem{lemma}{Lemma}[section]
\newtheorem{definition}{Definition}[section]

\newcommand\blfootnote[1]{%
	\begingroup
	\renewcommand\thefootnote{}\footnote{#1}%
	\addtocounter{footnote}{-1}%
	\endgroup
}

\newcommand\bbR{\ensuremath{\mathbb{R}}} 
\newcommand\bbN{\ensuremath{\mathbb{N}}} 
\newcommand\bbI{\ensuremath{\mathbb{I}}}

\newcommand\mY{\ensuremath{\mathcal{Y}}}
\renewcommand{\b}[1]{\ensuremath{\overline{#1}}}
\renewcommand{\v}[1]{\ensuremath{\boldsymbol{#1}}}

\newcommand{\floor}[1]{\lfloor #1 \rfloor}
\newcommand{\ceil}[1]{\lceil #1 \rceil}

\newcommand{\ul}[1]{\ensuremath{\underline{#1}}}

\title{Designing Optimal Binary Rating Systems}

\author{
	Nikhil Garg\\
	Stanford University\\
	\texttt{nkgarg@stanford.edu} \\
	\and
	Ramesh Johari\\
	Stanford University\\
	\texttt{rjohari@stanford.edu} \\
}

\begin{document}

\maketitle

\begin{abstract}
Modern online platforms rely on effective rating systems to learn about items.  We consider the optimal design of rating systems that collect binary feedback after transactions.  We make three contributions.  First, we formalize the performance of a rating system as the speed with which it recovers the true underlying ranking on items (in a large deviations sense), accounting for both items' underlying match rates and the platform's preferences.  Second, we provide an efficient algorithm to compute the binary feedback system that yields the highest such performance.  Finally, we show how this theoretical perspective can be used to empirically design an implementable, approximately optimal rating system, and validate our approach using real-world experimental data collected on Amazon Mechanical Turk.

\blfootnote{We thank Michael Bernstein and participants of the Market Design workshop at EC'18. This work was funded in part by the Stanford Cyber Initiative, the National Science Foundation Graduate Research Fellowship grant DGE-114747, and the Office of Naval Research grant N00014-15-1-2786.
}
\end{abstract}

\section{Introduction}

Rating and ranking systems are everywhere, from online marketplaces (e.g., 5 star systems where buyers and sellers rate each other) to video platforms (e.g., thumbs up/down systems on YouTube and Netflix). However, they are uninformative in practice~\citep{nosko_limits_2015}. One recurring pattern is that ratings \textit{binarize} -- most raters only use the extreme choices on the rating scale, and the vast majority of ratings receive the best possible rating. For example, 75\% of reviews on Airbnb receive a perfect rating of 5 stars~\citep{fradkin_determinants_2017}. Furthermore, several platforms have adopted a binary rating system, in which a user rates her experience as either positive or negative. Given the prevalence of binary feedback (either de facto or by design), in this work we investigate the optimal \textit{design} of such binary rating systems so that the platform can learn as fast as possible about the items being rated. 

The rating pipeline often works as follows: A buyer enters a platform and \textit{matches} with an item (e.g. selects a video on Youtube, is paired with a driver on Uber, or selects a home on AirBnB). She has an experience (e.g. a view, ride, or stay). Then, the platform asks her to \textit{rate} her experience, i.e. it asks her a question. In a binary system, she indicates whether her experience was positive or negative. She then leaves. The platform uses the ratings it has received to score the quality of items, potentially showing such scores to future buyers.

By \textit{designing} such a system, we mean: the platform can influence how the buyer rates -- how likely she is to give a positive rating, conditional on the quality of her experience. It can do so by asking her different questions, e.g. ``Was this experience above average'' or ``Was this experience the worst you've ever had?''. Different questions shift the probabilities at which items of various qualities receive positive ratings.

\textbf{Our first question is}: \textit{what is the structure of optimal binary feedback?} A rating system in which every buyer gives positive ratings after each match, independent of item quality, will fail to learn anything about the items. Clearly, better items should be more likely to receive positive ratings than worse ones. But how much more likely? 

Informally, suppose we have a set of items that match with buyers over time (at potentially differing rates), and we wish to rank the items by their true quality $\theta_i\in[0,1]$. The platform cannot observe $\theta_i$, however. Rather, in our model, after each match, an item with quality $\theta_i$ receives a positive rating with probability $\beta(\theta_i)$, and negative otherwise. In other words, the platform observes, for each item $i$, a sequence of ratings that are each Bernoulli$(\beta(\theta_i))$. Such ratings are the only knowledge the platform has about items. The platform ranks the items according to the percentage of its ratings (samples) that were positive. The function $\beta:[0,1]\mapsto[0,1]$ affects how quickly the platform learns the true ranking, and it prefers to maximize the learning rate. We show how to calculate an optimal $\beta$. 

As an example, consider three items with qualities $\theta_a > \theta_b > \theta_c$, and $\beta$ such that  $\beta(\theta_a) = 0.5$ and $\beta(\theta_c) = 0.1$, i.e. item $a$ gets positive ratings after $50\%$ of its matches, and item $c$ after $10\%$ of its matches. It is unclear what $\beta(\theta_b)$ should be. Trivially, $.1 < \beta(\theta_b) < .5$. Otherwise, even with infinitely many ratings the items will be mis-ranked.

But can we be more precise? If $\beta(\theta_b) = .49$, it will take many ratings of both items $a$ and $b$ to learn that $\theta_a > \theta_b$, but only a few from $c$ to learn that $\theta_c < \theta_b$. That may be good if the platform wants to identify the worst item, but not if it wants to identify the best. It may also be fine if items $a$ and $b$ match much more often with buyers than item $c$. Clearly, the optimal value for $\beta(\theta_b)$ is objective and context dependent. Of course, the problem becomes more challenging with more items $i$ for which $\beta(\theta_i)$ must be chosen. Lastly, in this example, one might intuitively think $\beta(\theta_b) = 0.3$ is optimal by symmetry when the items matter equally and matching rates are identical. This guess is incorrect. The optimal is $\beta(\theta_b) \approx 0.28$, due to the nature of binomial variance.  

In this work, we first formalize the above problem and show how to find an optimal $\beta(\theta)$, jointly for a set of items $[0,1]$. $\beta$ changes with the platform's objective and underlying item matching rates. Jumping ahead, Figure~\ref{fig:betadifferent} shows optimal $\beta$ in various settings under our model. For a platform that wants to find the worst sellers, for example, the top half of items should each get positive ratings at least $80+\%$ of the time; it is more important for the bottom half of items to be separated from one another, i.e. get positive ratings at differing percentages. 

Once we have calculated the optimal rating function $\beta$ (given context on the platform goals and matching rates), what should we do with it?


\textbf{Our second question is}: \textit{How does a platform build a rating system such that buyers behave near-optimally, i.e. according to a calculated $\beta$?} The platform cannot directly control buyer rating behavior. Rather, it has to ask \textit{questions} such that, for each item quality $\theta$, a fraction $\beta(\theta)$ of raters will give the item a positive rating. For example, by asking, ``Is this the best experience you've had,'' the platform would induce behavior such that $\beta(\theta)$ is small for most $\theta$. Most platforms today ask vague questions (e.g. thumbs up/down), and items mostly get positive ratings. We show this is highly suboptimal for ranking items quickly.

Our main contributions and paper outline are:

{\bf Rating system design as information maximization.}  In Sections~\ref{sec:modelbinary}-\ref{sec:asymptotic}, we formulate the design of rating systems as an information maximization problem.  In particular, a good rating system  recovers the true ranking over items, and converge quickly in the number of ratings.

{\bf Computing an optimal rating feedback function $\beta$.}  In Section~\ref{sec:discreteopt}, we develop an efficient algorithm that calculates the optimal rating function $\beta$, which depends on matching rates and the platform objective.  The optimal $\beta$ provides quantitative insights and principled comparisons between designs.

{\bf Real-world system design}.  In Section~\ref{sec:applic_insights}, we show how a platform can use a simple experiment and existing data to empirically design a near-optimal rating system, and to audit the current system.  In Section~\ref{sec:mturk}, we demonstrate the value of this approach through an experiment on Mechanical Turk.


\section{Related work}
\label{sec:relatedwork}

Many empirical and model-based works document and tackle challenges in existing rating systems~\citep{nosko_limits_2015,filippas_reputation_2017,zervas_first_2015,hu_overcoming_2009,cabral_dynamics_2010,tadelis_reputation_2016,cook_ubers_2015,rajaraman_five_2009,fradkin_determinants_2017,immorlica_emergence_2010,gaikwad_boomerang:_2016,bolton_engineering_2013}. To our knowledge, we are the first to formalize a rating system design problem and then show how one can use empirical data to optimize such systems. In a related paper \citep{garg_designing_2018}, we test behavioral insights using an experiment on a large online labor platform and develop a related design problem for a multiple-choice system, which proves far less tractable. 

Other works also optimize platform learning rates~\citep{ifrach_bayesian_2017,acemoglu_fast_2017,che_optimal_2015,johari_matching_2017,besbes_information_2018,papanastasiou_crowdsourcing_2017}. When prescriptive, they modify \textit{which matches occur}, while we view the matching process as given and modify \textit{the rating system}. The solutions are complementary. 

Many bandits works also seek to rank items from a sequence of observations~\citep{radlinski_learning_2008,yue_interactively_2009,katariya_dcm_2016,maes_automatic_2011}. Our problem is the \textit{inverse} of the bandit setting: given an arm-pulling policy, we design each arm's feedback.\footnote{Note that a {\em rating} is not the same as a {\em reward}; buyers often give positive ratings after bad experiences.} Our specific theoretical framework is similar to that of~\citet{glynn_large_2004}, who optimize a large deviations rate to derive an arm-pulling policy for best arm identification.


The ``twenty questions'' interpretation of Shannon entropy~\citep{cover_elements_2012,dagan_twenty_2017} seeks questions that can identify an item from its distribution.~\citet{dagan_twenty_2017} show how to almost match the performance of Huffman codes with only comparison and equality questions. Our work differs in two key respects: first, we seek to rank a set of items as opposed to identifying a single item; second, we consider non-adaptive policies (i.e. the platform cannot change its rating form in response to what it knows about an item already). 
\section{Model and optimization}
\label{sec:model_algo}

We now formalize our model and show how to optimize the rating function to maximize the learning rate. We focus on finding an optimal $\beta:[0,1] \mapsto [0,1]$, a map from item quality $\theta$ to the probability it should receive a positive rating. This section requires no data: we characterize the \textit{optimal} system.

\subsection{Model and problem specification}
\label{sec:modelbinary}

Our model is constructed to emphasize the rating system's learning rate. Time is discrete ($k = 0, 1, 2, \ldots$). Informally: there is a set of items. Each time step, buyers match with the items and leave a rating according to $\beta(\theta)$. The platform records the ratings and ranks the items. Formally:

{\bf Items}. The system consists of a set $[0,1]$ of items, where each item is associated with a unique (but unknown) quality $\theta \in [0,1]$; i.e., the system consists of a {\em continuum} of a unit mass of items whose unknown qualities are uniform\footnote{Any distribution can be handled by considering $\theta$ to be the item's \textit{quantile} rather than its absolute quality.} in $[0,1]$. Below, we \textit{discretize} the continuous quality space $[0,1]$ into $M$ types, to calculate a stepwise increasing $\beta$. We will make clear why we introduce a continuum but then discretize.

{\bf Matching with buyers}. Items accumulate ratings over time by matching with buyers.   We assume the existence of a nondecreasing {\em match function} $g(\theta)$, where item $\theta$ receives $n_k(\theta) = \lfloor k g(\theta) \rfloor$ matches, and thus ratings, up to time $k$. In other words, item $\theta$ is matched approximately every $\frac{1}{g(\theta)}$ time steps. $g(\theta)\leq 1$ and bounded away from $0$, i.e. $\exists c>0$: $g(\theta)>c$.  This accumulation captures the feature that better items may be more likely to match. 

{\bf Ratings}. The key quantity for our subsequent analysis is the {\em probability of a positive rating} for each $\theta$, $\beta(\theta) \triangleq \text{Pr}(\text{positive rating} | \theta)$.  
Let $y_\ell(\theta)\sim\text{Bernoulli}(\beta(\theta))$ be the rating an item of quality $\theta$ receives at the $\ell$th time it matches. 

{\bf Aggregating ratings and ranking sellers}. These ratings are aggregated into a {\em reputation score}, $x_k(\theta)$, at each time $k$. The score is the fraction of positive ratings received up to time $k$: $\label{eq:aggscore_bin}
x_k(\theta) \triangleq  \frac{1}{n_k(\theta)} \sum_{\ell = 0}^{n_k(\theta)} y_\ell(\theta)$ with $x_0(\theta) \triangleq 0$ for all $\theta$. Thus, $x_k(\theta) \sim \frac{1}{n_k(\theta)}\text{Binomial}(\beta(\theta), n_k(\theta))$.

{\bf System state}. The state of the system is given by a joint distribution $\mu_k(\Theta, X)$, which gives the \textit{mass} of items of quality $\theta \in \Theta \subset [0,1]$ with aggregate score $x_k(\theta) \in X \subset[0,1]$ at time $k$.  Because our model is a continuum, the evolution of the system state $\mu_k$ follows a deterministic dynamical system.

We have described these dynamics at the level of individual items; however, such statements should be interpreted as describing the evolution of the joint distribution $\mu_k$. The state update for $\mu_k$ is determined by the \textit{mass} of items that match and the \textit{distributions} of their ratings. A formal description of the state evolution is in Appendix Section~\ref{sec:appformalspec}.


{\bf Platform objective}.  The platform wishes to rank the items accurately. Given $\beta$ and $\theta_1 > \theta_2$, define:
\iftoggle{singlecolumnequations}
{\begin{align}
	\label{eq:Pk}
	P_k(\theta_1, \theta_2 | \beta) =& \mu_k\big(x_k(\theta_1) > x_k(\theta_2)|\theta_1, \theta_2\big) - \mu_k\big(x_k(\theta_1) < x_k(\theta_2)|\theta_1, \theta_2\big)
	\end{align}
}%
{\begin{align}
	\label{eq:Pk}
	P_k(\theta_1, \theta_2 | \beta) =& \mu_k\big(x_k(\theta_1) > x_k(\theta_2)|\theta_1, \theta_2\big)\nonumber \\&- \mu_k\big(x_k(\theta_1) < x_k(\theta_2)|\theta_1, \theta_2\big)
\end{align}
}
This expression captures observed score ranking's accuracy. When $\theta_1 > \theta_2$ but $x_k(\theta_1) < x_k(\theta_2)$, the ranking mistakenly orders $\theta_1$ below $\theta_2$.  A good system has large $P_k(\theta_1, \theta_2 | \beta)$. Integrating across items creates the following objective for each time $k$:
\begin{equation}
\label{eq:objective}
W_k = \int_{\theta_1 > \theta_2} w(\theta_1, \theta_2) P_k(\theta_1, \theta_2 | \beta) d\theta_1 d\theta_2
\end{equation}
Weight function $w(\theta_1, \theta_2) > 0$ indicates how much the platform cares about not mistaking a quality $\theta_1$ item with a quality $\theta_2$ item. We consider scaled $w$ such that $\int_{\theta_1 > \theta_2} w(\theta_1, \theta_2) d\theta_1 d\theta_2 = 1$.

Our first question then becomes: {\em What $\beta$ yields the highest value of $W_k$?}  As discussed above, the platform influences $\beta$ through the design of its rating system.  The optimal choice of $\beta$ sets the benchmark.

{\bf Discussion}. \textit{Objective function}. The specification \eqref{eq:objective} of the objective is quite rich.   It contains scaled versions of Kendall's $\tau$ (with $w(\theta_1, \theta_2) = 1$ for all $\theta_1, \theta_2$) and  Spearman's $\rho$ (with $w(\theta_1, \theta_2) = \theta_1 - \theta_2$) rank correlations. $w$ allows the platform to encode, for example, that it cares more about correctly ranking just the very best, very worst, or items at both extremes.\footnote{We use $\theta_1\theta_2(\theta_1 - \theta_2)$, $(1 - \theta_1)(1 - \theta_2)(\theta_1 - \theta_2)$, and $(\frac{1}{2} -\theta_1)^2(\frac{1}{2} -\theta_2)^2(\theta_1 - \theta_2)$ as examples. } \citet{tarsitano_comparing_2009} and \citet{da_costa_limit_2006} discuss other well-studied examples.

\textit{Relationship between model components.} Qualitatively, $\beta$ affects $W_k$ as follows, as previewed in the introduction: when $\beta(\theta_1) \approx \beta(\theta_2)$, then $x_k(\theta_1) \approx x_k(\theta_2)$, and so $P_k(\theta_1, \theta_2|\beta)$ is small (errors are common). A good design thus would have large $\beta(\theta_1) - \beta(\theta_2)$ for $\theta_1>\theta_2$ where $w(\theta_1,\theta_2)$ is large. Matching function $g$ also affects $P_k$ and thus $W_k$: when $g(\theta)$ is large, more ratings are sampled from item of quality $\theta$, i.e. $n_k(\theta)$ is higher, and so $x_k(\theta)$ is more closely concentrated around its mean $\beta(\theta)$.  Thus, $P_k(\theta, \theta'| \beta)$ increases (for all $\theta'$) with $g(\theta)$. A good design of $\beta$ thus considers both $w$ and $g$. 

\textit{{Matching}}. As noted above, we assume items receive a non-decreasing number of ratings based on their true quality, through matching function $g(\theta)$.  This is a reasonable approximation for our analysis, where we focus on the asymptotic rate of convergence of the ranking based on  to the true ranking, as the number of ratings increases.  In practice, items will be more likely to match when they have a higher {\em observed} aggregate score. Similarly, our model makes the stylized choice that all items have the same age. In reality, items have different ages in platforms. 

\textit{Non-response}. In practice, many buyers choose not to rate items, which our model does not capture. One possible approach is to treat non-response as a bad experience, which yields more information in the work of~\citet{nosko_limits_2015}. Solutions to non-response is an important area of work.

\begin{figure*}[t]
	\centering
	\begin{subfigure}[t]{.5\linewidth}
		\centering
		\includegraphics[width=.8\linewidth]{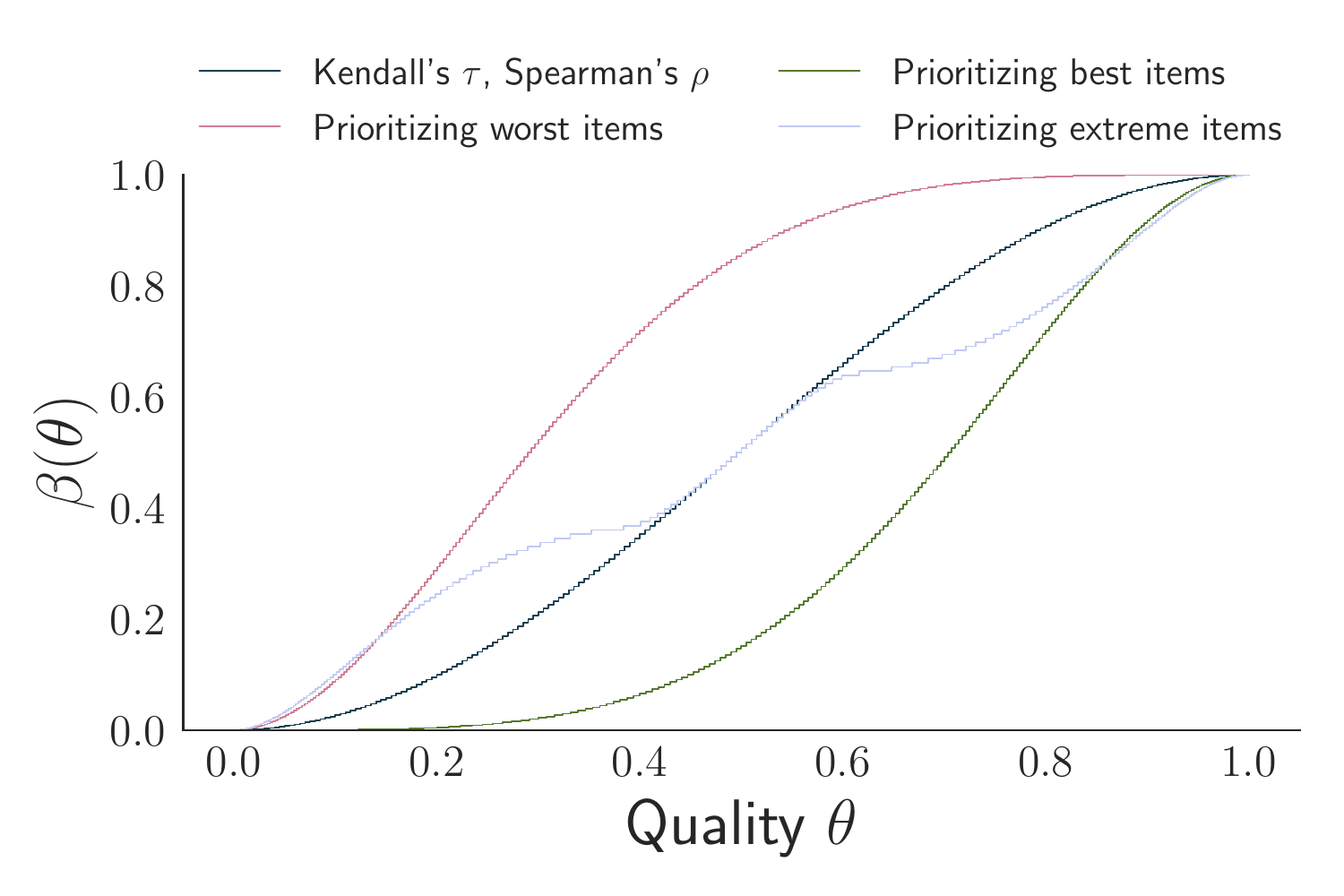}
		\caption{Fix $g=1$, with various objective function weights $w$}
		\label{fig:betadifferentweight}
	\end{subfigure}\hfill
	\begin{subfigure}[t]{.5\linewidth}
		\centering
		\includegraphics[width=.8\linewidth]{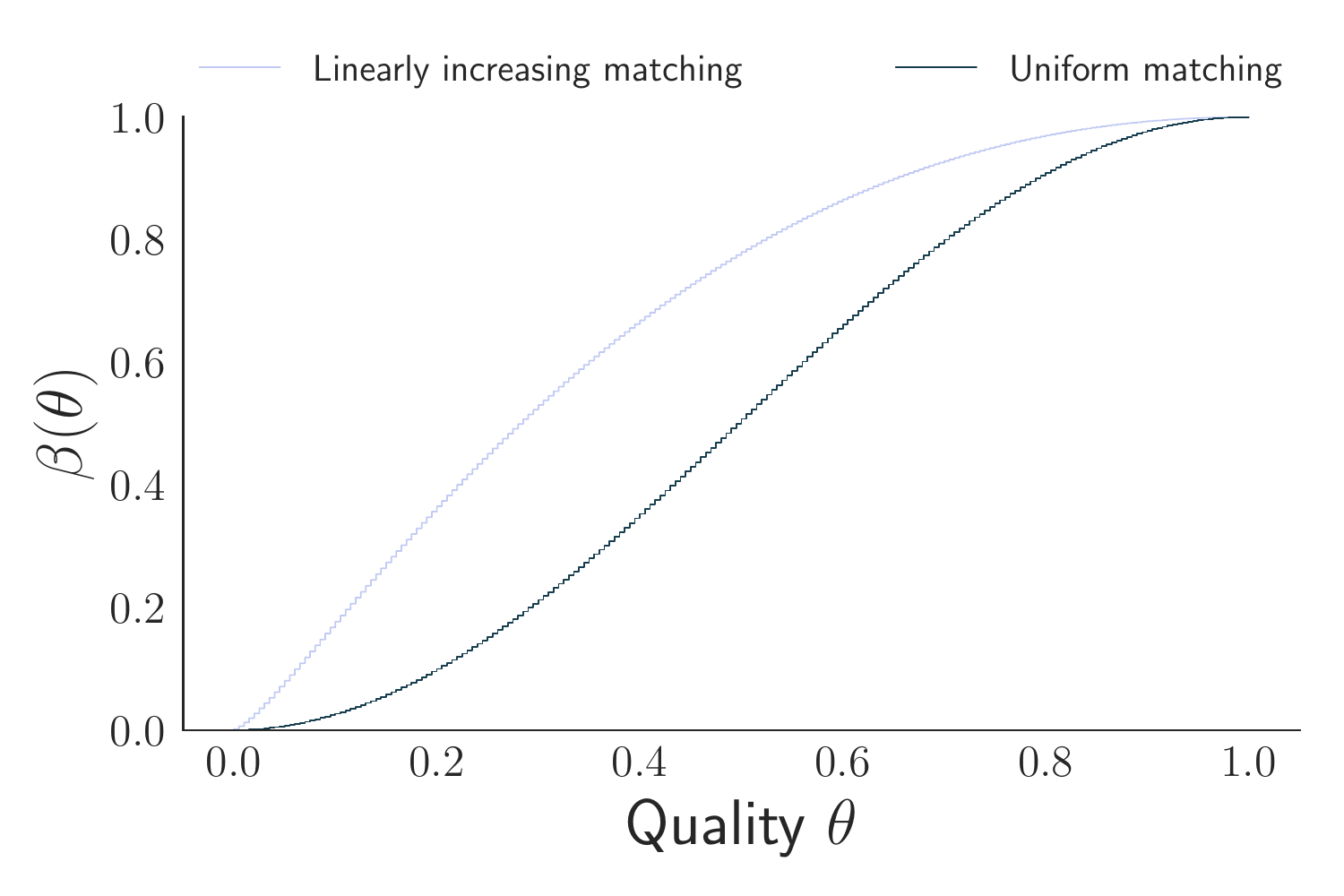}
		\caption{Fix $w = 1$. Vary matching, $g = 1$, and $g = \frac{1 + 10\theta}{11}$}
		\label{fig:betadifferentmatching}
	\end{subfigure}\hfill
	\caption{Optimal $\beta$ (with $M=200$) with various objective weight functions $w$ and matching rates $g$.} 
	\label{fig:betadifferent}
\end{figure*}

\subsection{Large deviations \& discretization}
\label{sec:asymptotic}

Recall the question: {\em What $\beta$ yields the highest value of $W_k$?}. We now refine objective $W_k$ and constrain $\beta$ to form a non-degenerate, feasible optimization task.

\textbf{Large deviation rate function}. $W_k$ is not one objective: it has a different value per time $k$, and no single $\beta$ simultaneously optimizes $W_k$ for all $k$.\footnote{For example, consider $\beta$ such that the worst half of items never receive a positive rating and the rest always do. It would perform comparatively well for a small number of ratings $k$, as it quickly distinguishes the best from the worst items. However it would never distinguish items within the same half. Some $\beta'$ may make more mistakes initially but perform better at larger $k$. }  
Considering \textit{asymptotic} performance is also insufficient: when $\beta$ is strictly increasing in $\theta$, $\lim_{k \to \infty} P_k(\theta_1, \theta_2 | \beta) = 1\,\,\,\forall \theta_1,\theta_2$ by the law of large numbers. Thus, $W \triangleq \lim_{k \to \infty} {W}_k = 1$, and {\em any} such $\beta$ is asymptotically optimal. 

For this reason, we consider maximization of the {\em rate} at which ${W}_k$ converges, i.e., how \textit{fast} the estimated item ranking converges to the true item ranking.  We use a large deviations approach~\citep{dembo_large_2010} to quantify this convergence rate.  Formally, given sequence $Y_k \leq \lim_{k\to\infty} Y_k = Y$, the {\em large deviations rate of convergence} is $-\lim_{k\to\infty} \frac{1}{k} \log (Y - Y_k) = c$. If $c$ exists then $Y_k$ approaches $Y$ exponentially fast: $Y - Y_k = e^{-kc + o(k)}$.

Then, we wish to {\em choose $\beta$ to maximize $W_k$'s large deviations rate}, $r = - \lim_{k \to \infty} \frac{1}{k} \log (W- {W}_k).$

\textbf{Discretizing $\beta$}. Unfortunately, even this problem is degenerate if we consider continuous $\beta$: for any $\beta$ that is not piecewise constant, the large deviations rate of convergence is zero, i.e., convergence of $W_k$ to its limit is only polynomially fast, and characterizing the dependence of this convergence rate on $\beta$ is intractable. Thus, the rate of convergence for $W_k$ is not  a satisfactory objective with continuous $\beta$.

We make progress by {\em discretizing} $\beta$; in particular, we restrict attention to optimization over stepwise increasing $\beta$ functions.\footnote{Note that, even for purely computational reasons, calculating $\beta$ requires discretization.}  Among stepwise increasing $\beta$, the large deviations rate of $W_k$ to its limiting value $W$ can be shown to be nondegenerate, i.e. $\exists c>0 $ s.t. $W - W_k = e^{-kc + o(k)}$.  (See Lemma \ref{lem:nolargedeviationsrate} in the Appendix for further discussion.)

Notationally, we will calculate an optimal stepwise increasing $\beta$ with $M$ levels, i.e. there are $M$ intervals $S_i \subset [0,1]$ and levels $t_i$ such that when $\theta_1, \theta_2 \in S_i$, then $\beta(\theta_1) = \beta(\theta_2) \triangleq t_i$. The challenge is calculating an optimal $\v{S^*} = \{S_i\}$ and $\v{t^*} = \{t_i\}$. 

The physical interpretation is that we group the items into $M$ subsets (types) $S_i \subset [0,1]$. When items $\theta_1, \theta_2$ are in the same subset, then their asymptotic reputation scores are the same, $\lim_k x_k(\theta_1)  = \lim_k x_k(\theta_2) \triangleq t_i$. These items cannot be distinguished from one another even asymptotically.

Though discretization allows us to define a large deviations rate for $W_k$, it comes at a cost: $W$, the limiting value of $W_k$, is no longer one. Different discretization choices $\v{S}$ result in different $W$.

\textbf{Our optimization problem:} \textit{Within the class of stepwise increasing functions with $M$ levels, find the $\beta$ that is \textit{optimal}, i.e. is }

	(1) {\em Asymptotically optimal}. It yields the highest limiting value of $W_k$. AND
	
	(2) {\em Rate optimal}. It yields the fastest large deviations rate $r$ among \textit{asymptotically optimal} $\beta$. 

A remarkable result of our paper is a $\mathcal{O}\left(M \log^2 \frac{M}{\epsilon} \right)$ procedure to find an optimal $\beta$ with $M$ levels.


\subsection{Solving the optimization problem} 
\label{sec:discreteopt}

The theorem below shows that the problem decomposes into two stages: first, find optimal discretization intervals $\v{S^*}$; then, find optimal $\v{t^*}$ given $\v{S^*}$.

\begin{restatable}{theorem}{lemCanMaximizeAsymValueAndRate}
		\label{thm:CanMaximizeAsymValueAndRate}
 The $\beta$ defined by the following choices of $\v{S}^*$ and $\v{t}^*$ is optimal:

 $\v{S}^* = \arg\max \sum_{0  \leq i < j < M} \int_{\theta_2\in S_i, \theta_1\in S_j } w(\theta_1, \theta_2)d(\theta_1,\theta_2)$
 
  $\v{t}^* = \arg\max r(\v{t})$, where\footnote{$\text{KL}(a||b) = a \log\frac{b}{a} + (1 - a) \log \frac{1-b}{1-a}$ is the Kullback-Leibler (KL) divergence between Bernoulli random variables with success probabilities $a$ and $b$ respectively.} $g_i \triangleq \inf_{\theta\in S_i} g(\theta)$ and
		\begin{align}
		r(\v{t}) 
		&\triangleq - \lim_{k \to \infty} \frac{1}{k} \log (W - {W}_k) \\
		&= \min_{0 \leq i \leq M-2} \inf_{a \in \bbR} \left \{ g_{i+1} \text{KL}(a || t_{i+1}) + g_i \text{KL}(a||t_i)\right\} \nonumber
		\end{align}	
\end{restatable}
  The proof is in the Appendix. The main hurdle is showing that the continuum of rates for $P_k(\theta_1,\theta_2)$ for each pair $\theta_1,\theta_2$ translates into a rate for $W_k$. This decomposition separates our two questions: $\v{S}^*$ maximizes the limiting value of $W_k$ \textit{given any} $\v{t}$, and depends only on $w$; Then, $\v{t}^*$ maximizes the rate at which the limiting value is reached, given $g_i$.

For Kendall's $\tau$ and Spearman's $\rho$, the optimal intervals are simply equispaced in $[0,1]$, i.e. $S^*_i = [\frac{i}{M}, \frac{i+1}{M})$, because the entire item quality distribution is equally important. For other objective weight functions $w$, the difficulty of finding the optimal subsets $\v{S}^*$ depends on the properties of $w$. Since $\v{S^*}$ is trivial for Kendall's $\tau$ and Spearman's $\rho$ -- and $w$ is just an analytic tool that formalizes a platform's goals -- we focus on finding the optimal levels $\v{t}^*$.

\textbf{Discussion}. One may naturally wonder why we introduced a continuum of quality $[0,1]$ and then discretized into $M$ subsets, instead of starting with $M$ types. As established in Theorem~\ref{thm:CanMaximizeAsymValueAndRate}, \textit{how} we discretize (i.e. solving for $\mathbf{S^*}$) allows for optimization of different objective weight functions $w$; it determines \textit{which} items are most valuable to distinguish.

Suppose we started with a set of $M$ items. Then the only remaining challenge is to equalize the rates at which each item is separated from others: the large deviations rate is unaffected by the weight function $w$ (it does not appear in the simplification of $r(\v{t})$). In other words, \textit{given} a discrete set of $M$ items (equiv, given $\v{S^*}$), calculating the optimal $\v{t}$ is equivalent to solving a \textit{maximin} problem for the rates at which each type is distinguished from each of the others. Thus, the algorithm below also solves the \textit{inverse} bandits problem in which we wish to rank the $M$ arms, and we can choose the structure of the (binary) observations at each arm. 

We further note that the choice of $M$ is not consequential; in the Appendix Section~\ref{sec:discretization} we show that in an appropriate sense, a sequence of optimal $\beta_M$ for each $M$ converges as $M$ gets large.

\paragraph{Algorithm to find the optimal levels}
We now describe how to find $\v{t^*}$, the maximizer of $r(\v{t})$. 

The following lemma describes a system of equations to find the $\v{t}^*$ that maximizes ${r}$. It states that $\v{t^*}$ equalizes the rates at which each interval $i$ is separated from its neighbors. The proof involves manipulation of $r$ and convexity, and is in the appendix.

\begin{restatable}{lemma}{lemsystemofeq}
	\label{lem:systemofequations} The unique solution $\v{t}^*$ to the following system of equations maximizes ${r}(\v{t})$: 
	 \begin{align}
	 r(t_0, t_1) &= r(t_1, t_2) = \dots = r(t_{M-2}, t_{M-1})\label{eq:sysequationsoneline}\\
	 t_0 &= 0, t_{M-1} = 1\nonumber\\
	 \iftoggle{singlecolumnequations}
	 {r(t_{i-1}, t_i) &\triangleq -\log \bigg[\bigg(
	 	(1-t_{i-1})^{\frac{g_{i-1}}{g_{{i-1}}+ g_{i}}}(1-t_{i})^{\frac{g_i}{g_{{i-1}} + g_{i}}}
	 	+ {t_{i-1}}^{\frac{g_{i-1}}{g_{{i-1}} + g_{i}}}{t_{i}}^{\frac{g_i}{g_{{i-1}} + g_{i}}}\bigg)^{g_{{i-1}}+ g_{i}}\bigg]\nonumber 
	 }%
	 { r(t_{i-1}, t_i) \triangleq -\log \bigg[&\bigg(
	 	(1-t_{i-1})^{\frac{g_{i-1}}{g_{{i-1}}+ g_{i}}}(1-t_{i})^{\frac{g_i}{g_{{i-1}} + g_{i}}}\nonumber\\
	 	+ &{t_{i-1}}^{\frac{g_{i-1}}{g_{{i-1}} + g_{i}}}{t_{i}}^{\frac{g_i}{g_{{i-1}} + g_{i}}}\bigg)^{g_{{i-1}}+ g_{i}}\bigg]\nonumber 
	 }
	 \end{align}
\end{restatable}

We do not know of any algorithm that efficiently and provably solves such convex equality systems in general. However, we leverage some structure in our setting to develop an algorithm, \textit{NestedBisection}, with run-time and optimality guarantees. The efficiency of our algorithm results from the property that, given a rate, $t_i$ is uniquely determined by the value of either of the adjacent levels $t_{i-1}, t_{i+1}$, reducing an exponentially large search space to an almost linear one. Physically, i.e., we only need to separate each type of item from its neighboring types. 



Below we include pseudo-code. Akin to branch and bound, the algorithm proceeds via bisection on the optimal value of $t_{M-2}$. For each candidate value of $t_{M-2}$, the other values can be found using a sequence of bisection subroutines. These values approximately obey all the equalities in the system~\eqref{eq:sysequationsoneline} except the first. The direction of the first equality's violation reveals how to change the interval for the next outer bisection iteration.  

\begin{restatable}{theorem}{thmalgo}
	\textit{NestedBisection} finds an $\epsilon$-optimal $\v{t}$ in $\mathcal{O}\left(M \log^2 \frac{M}{\epsilon} \right)$ operations, where $\epsilon$-optimal means that ${r}(\v{t})$ is within additive constant $\epsilon$ of optimal. \label{thm:algfinds} 
\end{restatable}
The proof is in the appendix. The main difficulty is finding a Lipschitz constant $\epsilon(\delta)$ for how much the rate changes with a shift  $\delta$ in a level $t_i$. This requires lower bounding $t_1$ as a function of $M$. In practice, the algorithm runs instantaneously on a modern machine (e.g. for $M = 200$). 

 \begin{algorithm}
	\DontPrintSemicolon
	\KwData{Number of intervals $M$, match function $g$}
	\KwResult{$\beta$ levels, i.e. $\{t_0 \dots t_{M-1}\} $}
	\Fn{main ($M$, $\delta$, $g$)}{
		\While{Uncertainty region for $t_{M-2}$ is bigger than error tolerance}{ 
			Calculate $r(t_{M-2}, 1)$, the rate between current guess for $t_{M-2}$, and $t_{M-1} = 1$. \;
			Fixing $t_{M-2}$, find $t_1\dots t_{M-3}$ such that $r(t_{1}, t_2) = r(t_{2}, t_3)= \dots = r(t_{M-3}, t_{M-2}) = r(t_{M-2}, 1)$, which can be done through a sequence of bisection routines.\;
			Calculate $r(0, t_1)$, the rate between current guess for $t_{1}$, and $t_0 = 0$. \;
			Compare $r(t_{M-2}, 1)$ and $r(0, t_1)$, adjust uncertainty region for $t_{M-2}$ accordingly.\;
		}
		return $\{t_i\}$\;
	}
	\caption{Nested Bisection}
	\label{alg:Fasttimealgorithm}
\end{algorithm}

\subsection{Visualization and discussion}
\label{sec:designinsights}
Figure~\ref{fig:betadifferent} shows how the optimal $\beta$ varies with weights $w$ and matching rates $g$. Higher relative weights in a region lead to a larger range of $\beta(\theta)$ there to make it easier to distinguish those items (e.g., prioritizing the best items induces a $\beta$ shifted right). Higher relative matching rates $g(\theta)$ have the opposite effect, as frequent sampling naturally increases accuracy for the best items. We formalize this shifting in Appendix Section~\ref{sec:objectivematchingeffect}. 

It is interesting that even the basic case, with $w = 1$ and $g = 1$, has a non-trivial $\beta$. One would expect, with weight and matching functions that treat all items the same, that $\beta$ would be linear, i.e. $\beta(\theta) = \theta$. Instead, a third factor non-trivially impacts optimal design: binomial variance is highest near $\beta(\theta) = \frac{1}{2}$. Items that receive positive ratings at such frequency have high-variance scores, and thus the optimal $\beta$ has a smaller mass of items with such scores. 


\section{Designing approximately optimal, implementable rating systems}
\label{sec:applic_insights}

We now turn to our second question: \textit{How does a platform build and implement a real rating system such that buyers behave near-optimally, i.e. according to a calculated $\beta$?}. In this section, we give an example design procedure for how a platform would do so, and in the next section we validate our procedure through an experiment on Mechanical Turk.

Recall that $\beta(\theta)$ gives the probability at which an item of quality $\theta$ should receive a positive rating. However, the platform cannot force buyers to rate according to this function. Rather, it must ask \textit{questions} of buyers that will induce a  proportion $\beta(\theta)$ of them to give positive ratings for an item $\theta$.

Throughout the section, we assume that an optimal $\beta(\theta)$ has been calculated (for some $M$, $g$, and $w$). 

\textbf{Resources available to platform}. We suppose the platform has a set $\mY$ of possible binary questions that it could ask a buyer, e.g., ``Are you satisfied with your experience'' or ``Is this experience your best on our platform?''. Informally, at each rating opportunity (i.e., match made), the rater can be shown a single question $y\in\mY$. Let $\psi(\theta, y)$ be the probability an item quality $\theta$ would receive a positive rating when the rater is asked question $y\in\mY$.

We further suppose the platform has a set $\Theta$ of representative items for which it has access to item quality.  $\Theta$, then, is the granularity at which the platform can collect data about historical performance, or otherwise get expert labels. (We assume $M\gg|\Theta|$).

 Using known set $\Theta$, the platform can run an experiment to create an estimate $\hat\psi(\theta, y), \forall y\in\mY, \theta\in\Theta$.

\textbf{Design heuristic}. How can the platform build an effective rating system using these primitives and $\beta$? We consider the following heuristic design: the platform randomly shows a question $y\in\mY$ to each buyer. The choice of the platform is a {\em distribution} $H(y)$ over $y \in \mY$; in other words, for the platform the design of the system amounts to choosing the frequency with which each question is shown.  At each rating opportunity, $y$ is chosen from $\mY$ according to $H$, independently across opportunities.  

 Clearly, the probability that an item $\theta$ receives a positive rating is:
$\tilde{\beta}(\theta)\triangleq\sum_{y\in\mY} \psi(\theta,y) H(y).$

We want a distribution $H$ such that $\tilde{\beta}(\theta) = \beta(\theta)$ for all $\theta \in [0,1]$, i.e., that the positive rating probability for each item is exactly the optimal value.  However, there may not exist {\em any} set of questions $\mY$ with associated $\psi$ and choice of $H$ such that $\tilde{\beta} = \beta$.\footnote{There are special cases where an exact solution exists.  For example, let $\mY = [0,1]$, and $\psi(\theta, y) = \bbI[\theta\geq y]$. 
}

We propose the following heuristic to address this difficulty. Choose a probability distribution $H$ to minimize the following $L_1$ distance:
\begin{equation} \min_{H: \|H\|_1 = 1}\sum_{\theta \in \Theta} |\beta(\theta) - \sum_{y\in\mY} \hat\psi(\theta,y) H(y)| \label{eqn:l1heuristic} \end{equation}
This heuristic uses the data available to the platform, ${\hat\psi(\theta,y)}$ for a set of items $\theta \in \Theta$, and designs $H$ so at least these items receive ratings close to their optimal ratings $\beta(\theta_i)$. Then, as long as $\psi$ is well-behaved, and $\Theta$ is representative of the full set, one can hope that $\tilde{\beta}(\theta)\approxeq \beta(\theta)$, for all $\theta \in [0,1]$.

\textbf{Discussion}. 
\textit{Real-world analogue \& constraints}. A special case of this system is already in place on many platforms, where the same question is always shown. Static systems can be designed by restricting $H$ to only have mass at one question $y$. More generally, constraints can be used in optimization~\eqref{eqn:l1heuristic}. 

\textit{Limitations}. Our model does not allow for $y$ to be chosen adaptively based on the platform's current knowledge of the item. In practice, this may be a reasonable restriction for implementation purposes. Our model also restricts aggregation to be binary; the platform in our model does not use information on how ``hard'' a question $y$ is.

\section{Mechanical Turk experiment}
\label{sec:mturk}
In the previous sections, we showed how to find an optimal rating function $\beta$ and we how to apply such a $\beta$ to design a binary rating system using empirical data.  In this section, we deploy an experiment on Amazon Mechanical Turk to apply these insights in practice. First, we collect data that can be used  to create a reasonable real-world example of $\psi(\theta,y)$, as a proof of concept with which we can apply our optimization approach. Then, we use this model to demonstrate some key features of optimal and heuristic designs as computed via our methodology, and show that they perform well relative to natural benchmarks. Details of experimental design, simulation methodology, and results are in Appendix~\ref{sec:appmturk}. 

\begin{figure*}[t]
	\centering
	\begin{subfigure}[t]{.34\linewidth}
		\includegraphics[width=\linewidth]{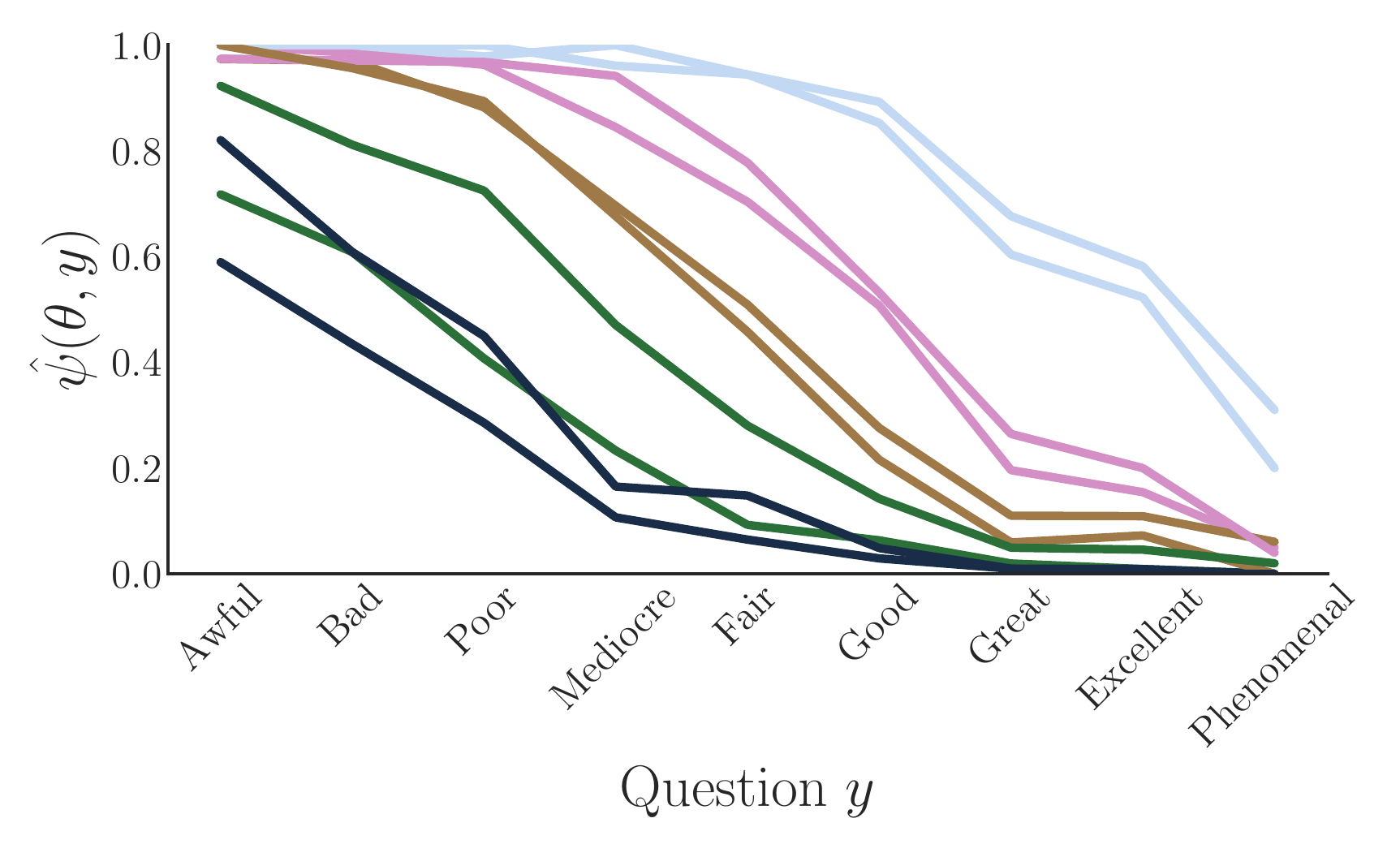}
		\caption{Experimental $\hat\psi(\theta_i, y)$. Light blue lines are best 2 paragraphs, dark blue the worst.}
		\label{fig:psi_main}
	\end{subfigure}\hfill
	\begin{subfigure}[t]{.32\linewidth}
		\includegraphics[width=\linewidth]{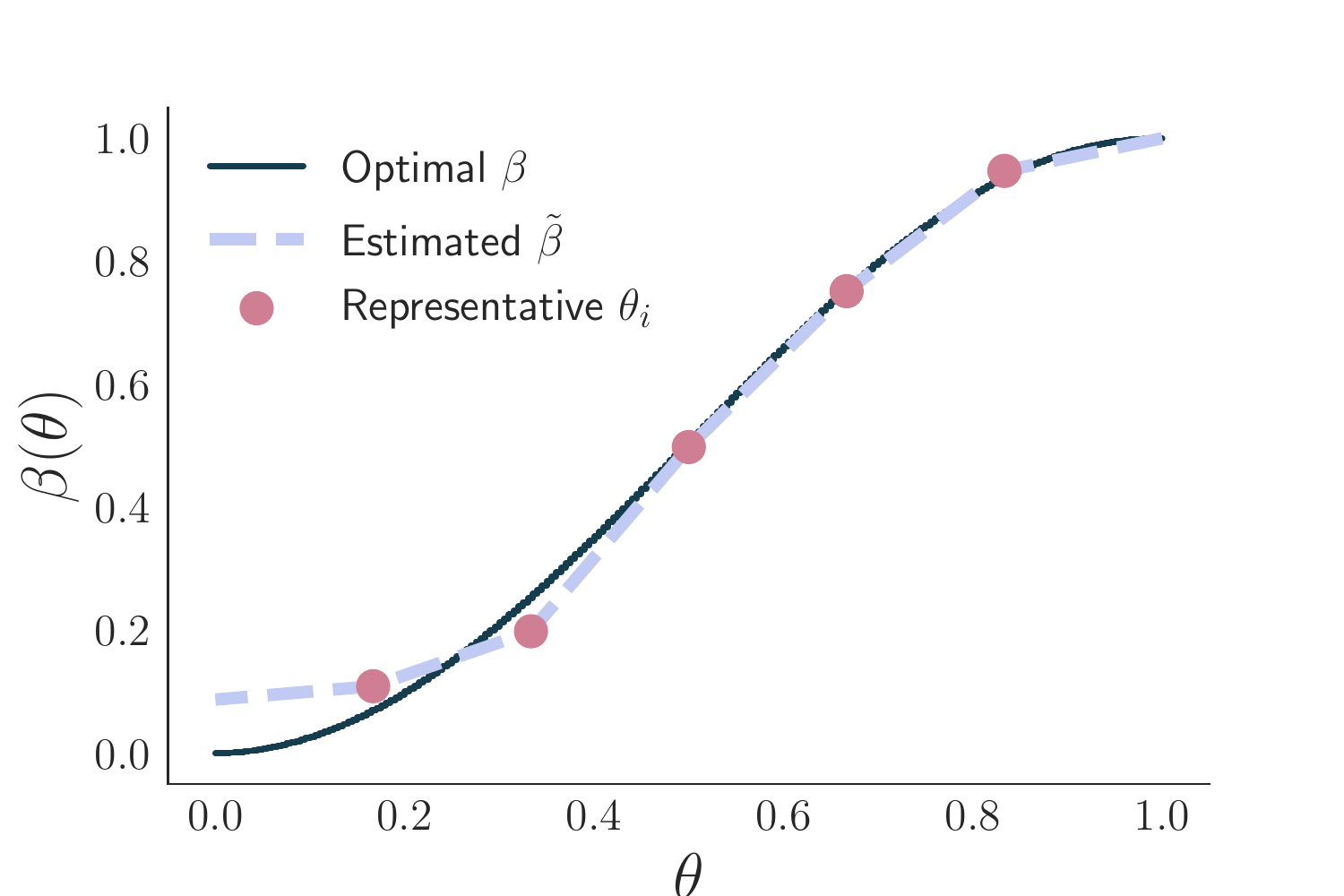}
		\caption{$\beta$ vs $\tilde{\beta}$ (using a calculated H) for $w=1, g=1$.}
		\label{fig:betavsbetatilde}
	\end{subfigure}\hfill
	\begin{subfigure}[t]{.31\linewidth}
		\includegraphics[width=\linewidth]{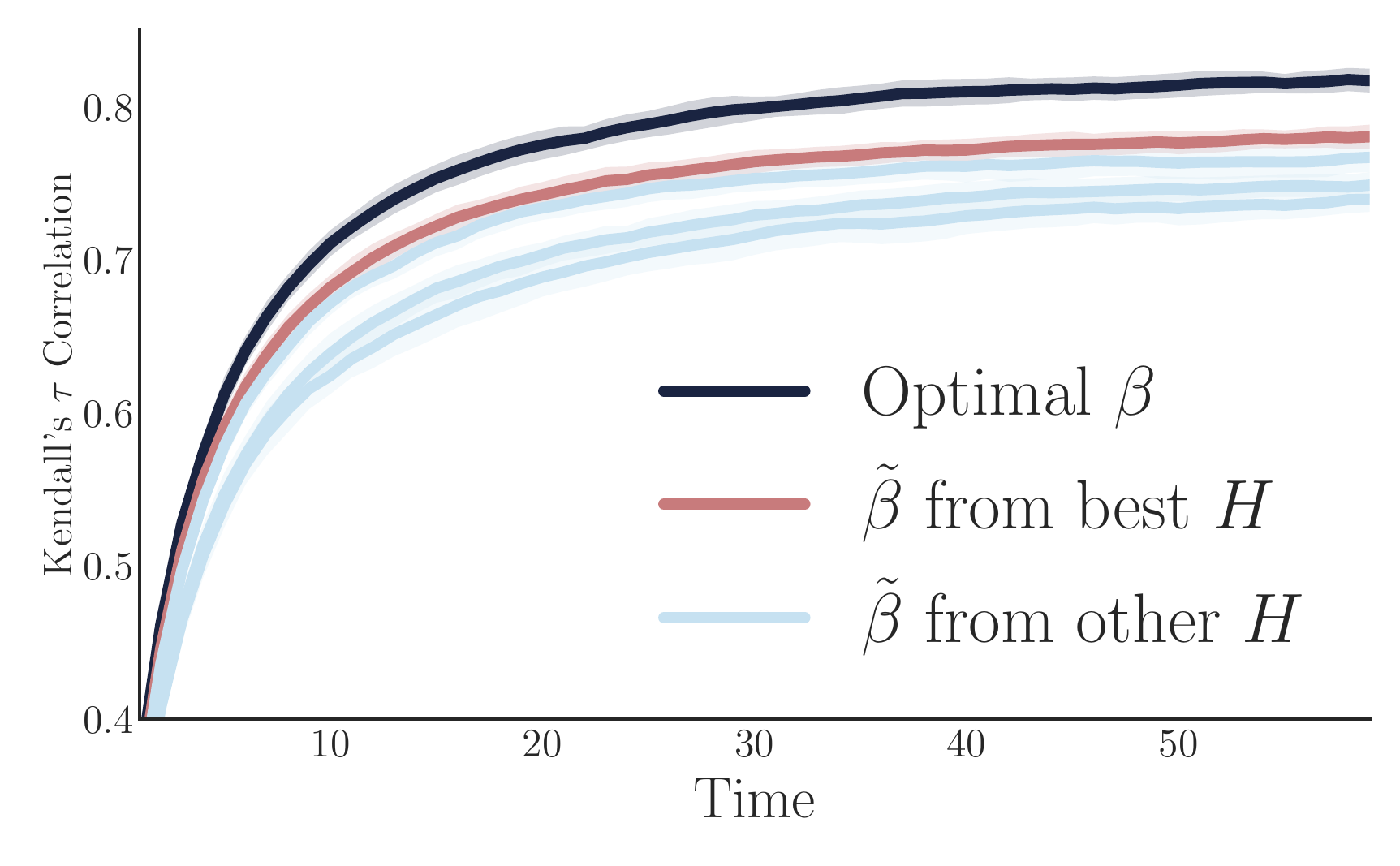}
		\caption{Simulated performance for $w=1, g=1$.}
		\label{fig:mturk_unifdeath02_main}
	\end{subfigure}\hfill	
	\caption{Experiment and simulation results} 
	\label{fig:experimentsall}
\end{figure*}

\textbf{Experiment description}. We have a set of 10 English paragraphs of various writing quality, with \textit{expert scores} $\theta$, from a TOEFL book~\citep{educational_testing_service_toefl_2005}; there were $5$ unique possible experts ratings, i.e. $\theta \in \{.1,.3,.5,.7,.9\}$. For each possible rating, we have two paragraphs who received that score from experts. 

We asked workers on Mechanical Turk to rate the writing quality of the paragraphs from a set of adjectives, $\mY$. Using this data, we estimate $\psi(\theta, y)$, i.e., the probability of positive rating when a question based on adjective $y \in \mY$ is shown for paragraph with quality $\theta\in\Theta$. (e.g., ``Would you consider this paragraph of quality \textit{[adjective]} or better?'') Figure~\ref{fig:psi_main} shows our estimated $\hat\psi$ for our 10 paragraphs.

\textbf{Optimization.} Next, we find the optimal $\beta$ for various matching and weight functions using the methods from Section~\ref{sec:model_algo}. In particular, we have $\beta$ for all permutations of the cases $g = 1, g = \frac{1 + 10\theta}{11}$, and $w = 1, w = \theta_1\theta_2(\theta_1 - \theta_2)$, and $w = (1 - \theta_1)(1 - \theta_2)(\theta_1 - \theta_2)$. Recall that this step does not use experimental data.

Then, using $\hat\psi$ and calculated $\beta$s, we apply the heuristic from Section~\ref{sec:applic_insights} to find the distribution $H$ with which to sample the questions (adjectives) in $\mY$. Figure~\ref{fig:betavsbetatilde} shows the optimal $\beta$ (with $g = 1$, $w = 1$), and  estimated $\tilde{\beta}$ from the procedure.

\textbf{Simulation.} Finally, we study the performance of these designs via simulation in various settings. We simulate a system with 500 items and 100 buyers according to the model in Section \ref{sec:modelbinary}, except that matching is \textit{stochastic}: at each time, a random 100 items receive ratings, based on \textit{observed} scores $x_k(\theta)$ rather than true quality $\theta$. Furthermore, in some simulations, we have sellers \textit{entering and exiting} the market with some probability at each time step. We measure the performance of all the designs. For comparison, we also simulate a \textit{naive} $H = \frac{1}{|\mY|}$.

Note that our experiment only provides $\hat\psi$ associated with qualities $\theta \in \Theta$, and for simulations we construct a full $\psi(\theta, y), \forall \theta \in [0,1]$ from these points by averaging and interpolating (in order to model human behavior for the full system). Further, our \textit{calibrated} simulations only provide rough evidence for the approach: although we use real-world data, the simulations assume that our model reflects reality, except for where we deviate as described above. 

\textbf{Results and discussion}. 
Figure~\ref{fig:mturk_unifdeath02_main} shows the simulated performance (as measured by Kendall's $\tau$ correlation) of the various designs over time, when $g = 1$. Further plots are in the Appendix Figure~\ref{fig:mturksimulationsbinary}, showing performance under various weight functions $w$ and matching functions $g$, and with items entering and exiting the market. We find that:

	First, the optimal $\beta$ for each setting outperforms other possible functions, as expected. The designs are robust to (some) assumptions in the model being broken, especially regarding matching.
	
	Second, the $H$ from our procedure outperforms other designs, but is worse than the optimal system $\beta$. In general, the simulated gap between an implemented system and optimal design $\beta$ provides the platform quantitative insight on the system's sub-optimality. 
	
	Third, comparing $\tilde{\beta}$ to $\beta$ gives qualitative insight on \textit{how} to improve the system. For example, in Figure~\ref{fig:betavsbetatilde}, $\tilde{\beta}$ is especially inaccurate for $\theta\in[0,.4]$. The platform must thus find better questions for items of such quality. Figure~\ref{fig:psi_main} corroborates: our questions cannot separate two low quality paragraphs rated differently by experts (in dark blue and green). 
	

A wide range of recent empirical work has documented that real-world rating systems experience substantial inflation; almost all items receive positive ratings almost every match~\citep{filippas_reputation_2017,fradkin_determinants_2017,tadelis_reputation_2016}. Our formulation helps understand how -- and whether -- such inflation is suboptimal, and provides guidance to platform designers. In particular, rating inflation can be interpreted as a current $\beta(\theta)$ that is very high for almost all item qualities $\theta$.  This system is well-performing if the platform objective is to separate the worst items from the rest, or if high quality items receive many more ratings than low quality ones; it is clearly sub-optimal in other cases.  Our paper provides a template for how a platform might address such a situation.

\FloatBarrier
\begin{spacing}{0.9}
\bibliographystyle{plainnat}

\end{spacing}

\newpage
\pagebreak
\begin{appendices}
\section{Mechanical Turk experiment, simulations, and results}
\label{sec:appmturk}
\FloatBarrier

In this section, we expand upon the results discussed in Section~\ref{sec:mturk}. We design and run an experiment that a real platform may run to design a rating system. We follow the general framework in Section~\ref{sec:applic_insights}. We first run an experiment to estimate a $\psi(\theta, y)$, the probability at which each item with quality $\theta$ receives a positive answer under different questions $y$. Then, we design $H(y)$, using our optimal $\beta$ for various settings (different objectives $w$ and matching rates $g$). Then, we simulate several markets (using the various matching rates $g$) and measure the performance of the different rating system designs $H$, as measured by various objective functions~\eqref{eq:objective}. 

\subsection{Experiment description}
\label{sec:appmturkdescription}

We now describe our Mechanical Turk experiment. We ask subjects to rate the English proficiency of ten paragraphs. These paragraphs are modified TOEFL (Test of English as a Foreign Language) essays with known scores as determined by experts~\citep{educational_testing_service_toefl_2005}. Subjects were given six answer choices, drawn randomly from the following list: \textit{Abysmal, Awful, Bad, Poor, Mediocre, Fair, Good, Great, Excellent, Phenomenal}, following the recommendation of~\citet{hicks_choosing_2000}. \textit{Poor} and \textit{Good} are always chosen, and the other four are sampled uniformly at random for each worker.  One paragraph is shown per page; returning to modify a previous answer is not allowed; and paragraphs are presented in a random order. This data is used to calibrate a model of $\psi$ for optimization, i.e. to simulate a system with a set of questions $\mY$, where each question $y$ corresponds to a adjective, ``Would you characterize the performance of this item as [adjective] or better?''.\footnote{The data from the experiment is also used for a separate paper, \cite{garg_designing_2018}. In that work, we analyze the full multi-option question directly, but the main focus is reporting the results of a separate, live trial on a large online labor platform.}

Different experiment trials are described below. Pilots were primarily used to garner feedback regarding the experiment from workers (fair pay, time needed to complete, website/UI comments, etc). All trials yield qualitatively similar results in terms of both paragraph ratings and feedback rating distributions for various scales. 
\begin{description}
	\item [Pilot 1] 30 workers. Similar conditions as final experiment (6 words sampled for paragraph ratings, all uniformly at random, 5 point scale feedback rating), with identical question phrasing, ``How does the following rate on English proficiency and argument coherence?''.
	\item [Pilot 2] 30 workers. 7 words sampled for paragraph ratings, 6 point scale feedback rating, with the following question phrasing: ``How does the following person rate on English proficiency and argument coherence?''.
	\item [Experiment] 200 workers. 6 words sampled for paragraph ratings, with 2 fixed as described above, 5 point scale feedback rating. Question phrasing, ``How does the following rate on English proficiency and argument coherence?''.
\end{description}

\begin{figure}
	\centering
	\begin{subfigure}[t]{.33\linewidth}
		\includegraphics[width=\linewidth]{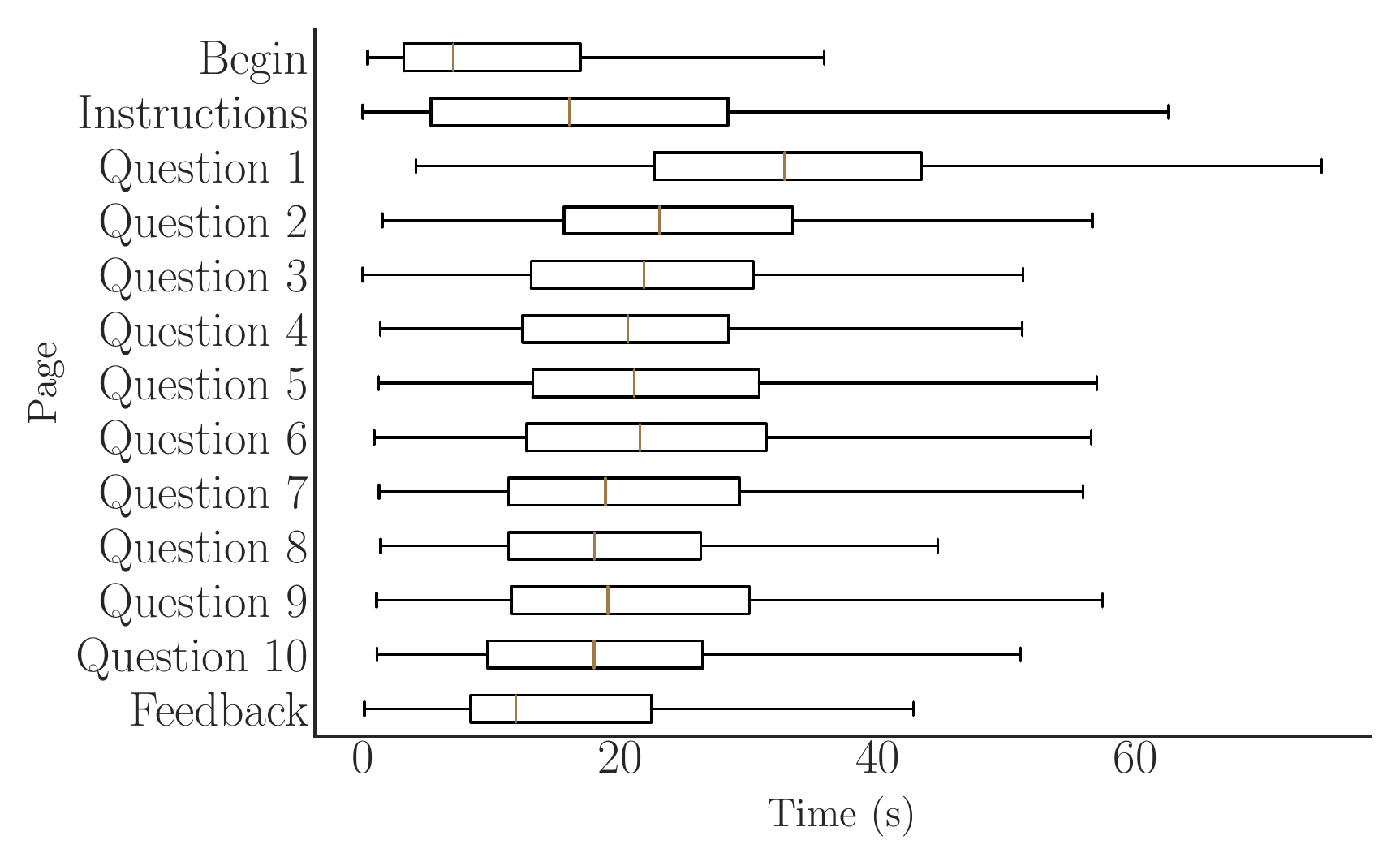}
		\caption{Time spent per page.}
		\label{fig:mturktimeperpage}
	\end{subfigure}\hfill
	\begin{subfigure}[t]{.33\linewidth}
		\includegraphics[width=\linewidth]{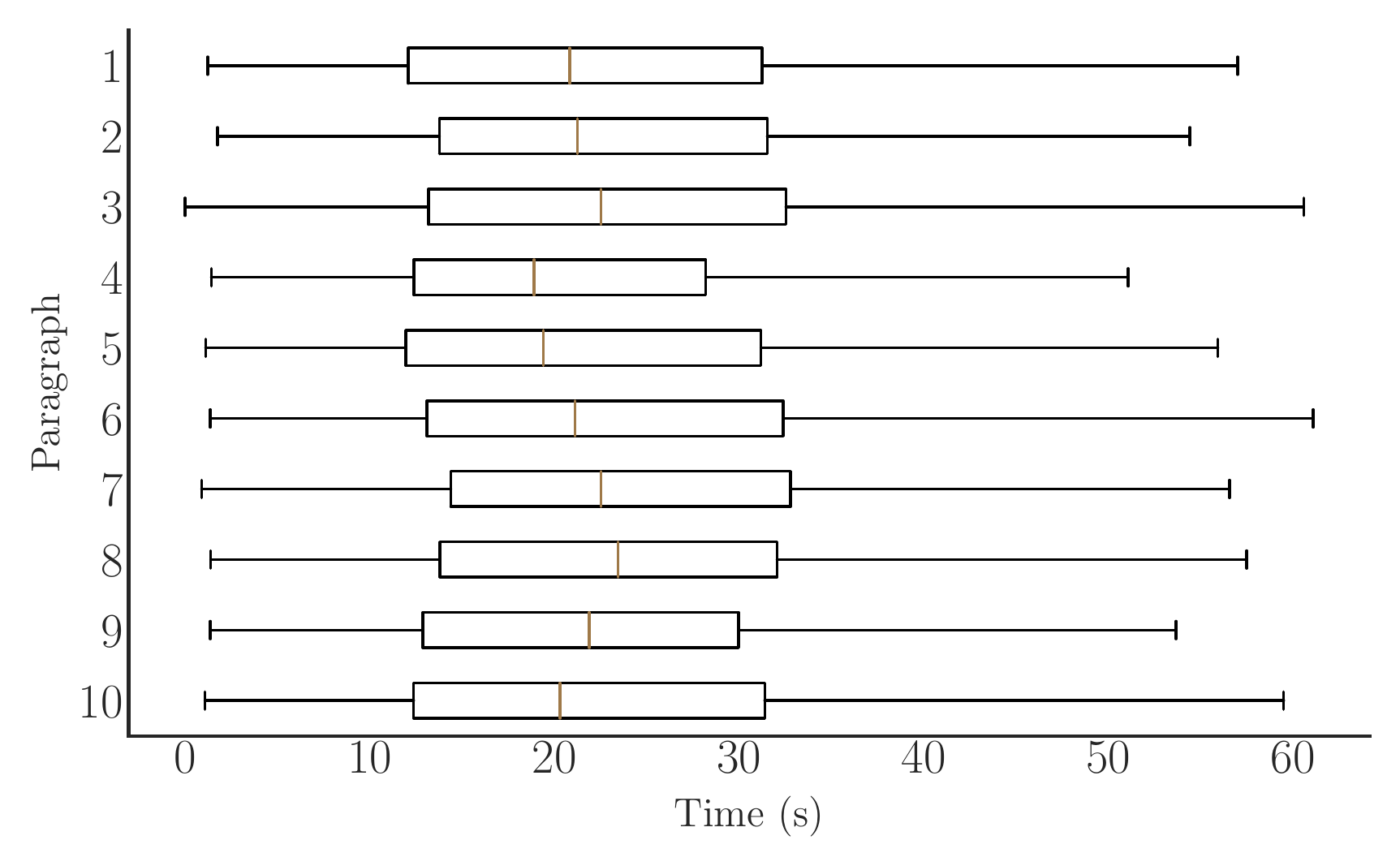}
		\caption{Time spent per paragraph.}
		\label{fig:mturktimeperpar}
	\end{subfigure}\hfill
	\begin{subfigure}[t]{.33\linewidth}
		\includegraphics[width=\linewidth]{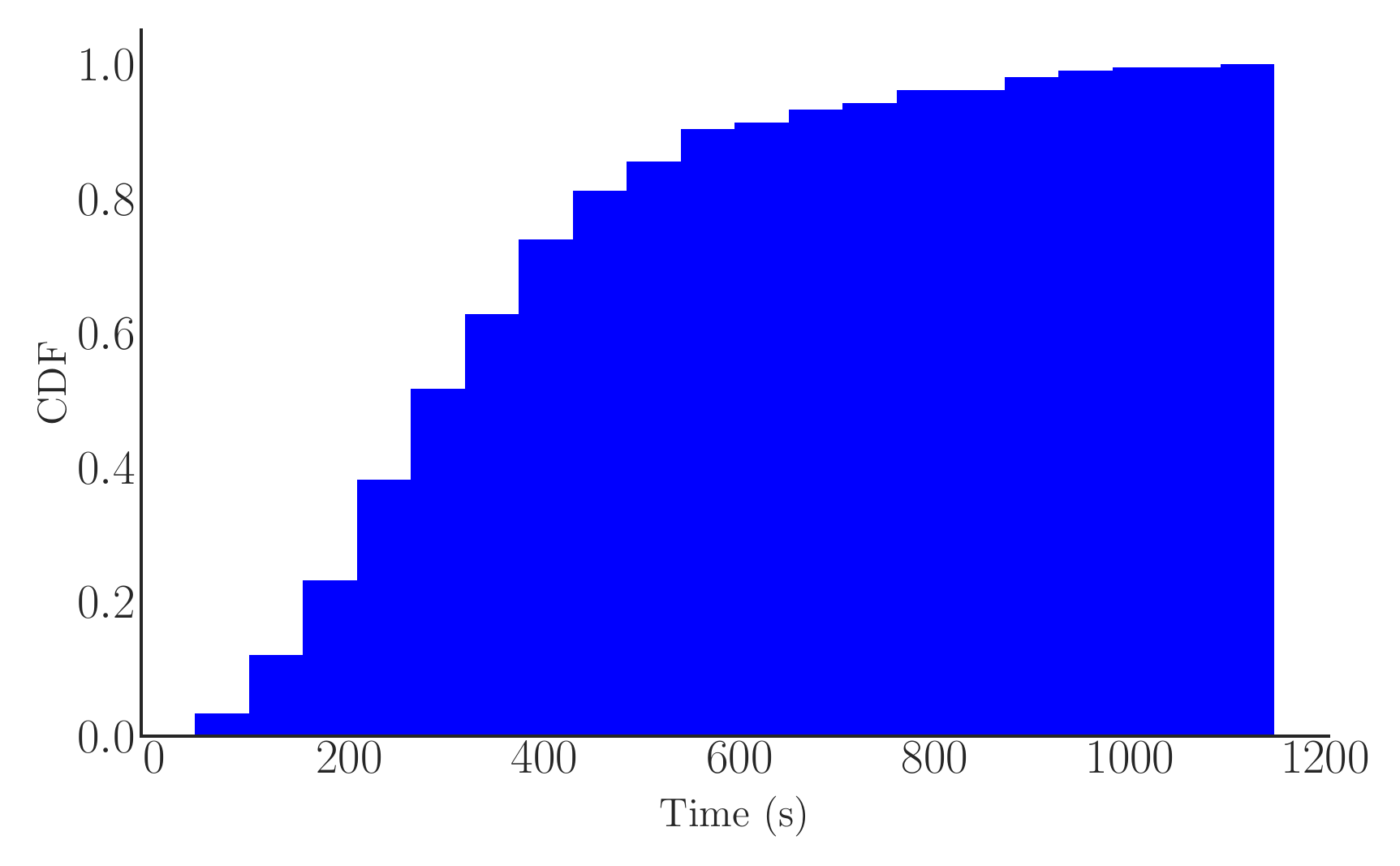}
		\caption{Distribution (CDF) of total time spent.}
		\label{fig:mturktimecdf}
	\end{subfigure}\hfill
	\caption{Additional information for MTurk experiment}
	\label{fig:mturkappendix}
\end{figure}

We use paragraphs modified from a set published by the Educational Testing Service~\citep{educational_testing_service_toefl_2005}. There are 10 paragraphs, 5 each on 2 different topics. For each topic, the paragraphs have 5 distinct expert scores. Paragraphs are shortened to just a few sentences, and the top rated paragraphs are improved and the worst ones are made worse, preserving the ranking according to the expert scores.

Figure~\ref{fig:mturktimeperpage} shows time spent on each page of the experiment, Figure~\ref{fig:mturktimeperpar} shows the time spent per paragraph, and Figure~\ref{fig:mturktimecdf} shows the cumulative density function for time spent by workers. The paragraphs are presented to workers in a random order. No workers are excluded in our data and all workers were paid $\$1.00$, including the ones that spent 2-3 seconds per page. $7/60$ workers in the pilots received a bonus of $\$0.20$ for providing feedback. The instructions advised workers to spend no more than a minute per question, though this was not enforced.

The instructions for the main experiment were as follows: ``Please rate on English proficiency (grammar, spelling, sentence structure) and coherence of the argument, but not on whether you agree with the substance of the text.'' No additional context was provided. 

\subsection{Calculating optimal $\beta$ and $H$}
\begin{figure}
	\centering
	\includegraphics[width=.5\linewidth]{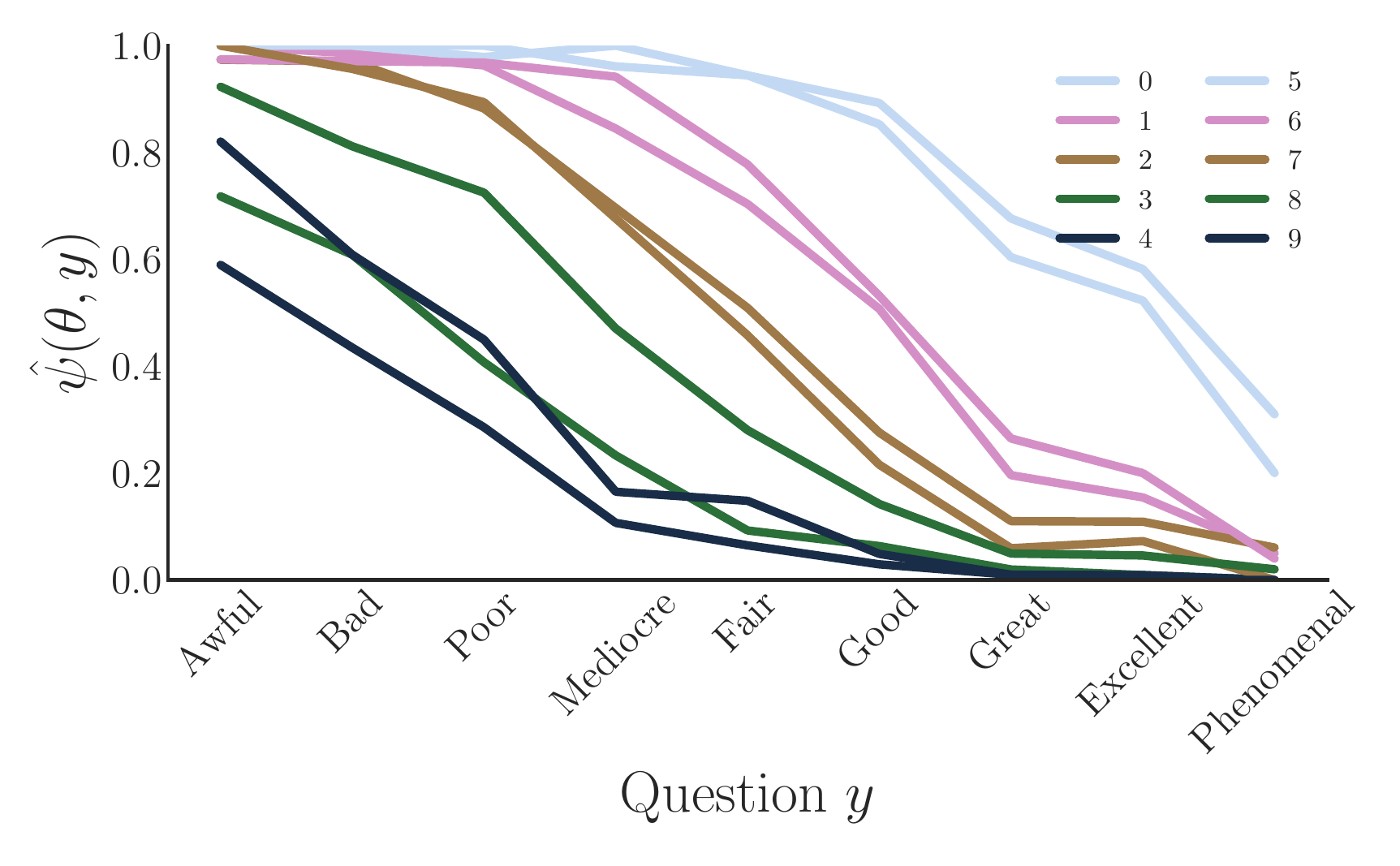}
	\caption{Paragraph rating distribution -- for paragraph $\theta$ and rating word $y$, the empirical $\hat\psi (\theta, y)$ is shown. Colors encode the true quality as rated by experts (light blue is best quality, dark blue is worst). }
	\label{fig:paragraphdist}
\end{figure}

Figure~\ref{fig:paragraphdist} shows the empirical $\hat\psi(\theta, y)$ as measured through our experiment. The colors encode the true quality as rated by experts (light blue is best quality, dark blue is worst); recall there are 10 paragraphs with 5 distinct expert ratings (paragraphs 0 and 5 are rated the best, paragraphs 4 and 9 are rated the worst). 

With the $\beta$ calculated and visualized using the methods in Section~\ref{sec:model_algo}, we now find the optimal $H$ for various settings using the methods in Section~\ref{sec:applic_insights}. We view our set of paragraphs as \textit{representative items} $\Theta$ from a larger universe of paragraphs. In particular, we view our worst quality paragraphs as in the $10$th percentile of paragraphs, and our best items as in the $90$th percentile. In other words, from the empirical $\hat\psi$, we carry out the methods in Section~\ref{sec:applic_insights} using a $\psi$ s.t. $\psi(.1, y) = (\hat\psi(4, y)+ \hat\psi(9, y)/2$ (and similarly for $\psi(.3, y),\psi(.5, y),\psi(.7, y),\psi(.9, y)$, where e.g. $\hat\psi(4, y)$ is the empirical rate at which paragraph 4 received a positive rating on question $y$. 

Then, we solve the optimization problem for $H$ stated in Section~\ref{sec:applic_insights}. From the above discussion, we want to find an $H$ such that the worst rated paragraphs in our experiment have a probability of receiving a positive rating that is approximately $\beta(.1)$.

\begin{figure}
	\centering
	\begin{subfigure}[t]{.48\linewidth}
		\includegraphics[width=\linewidth]{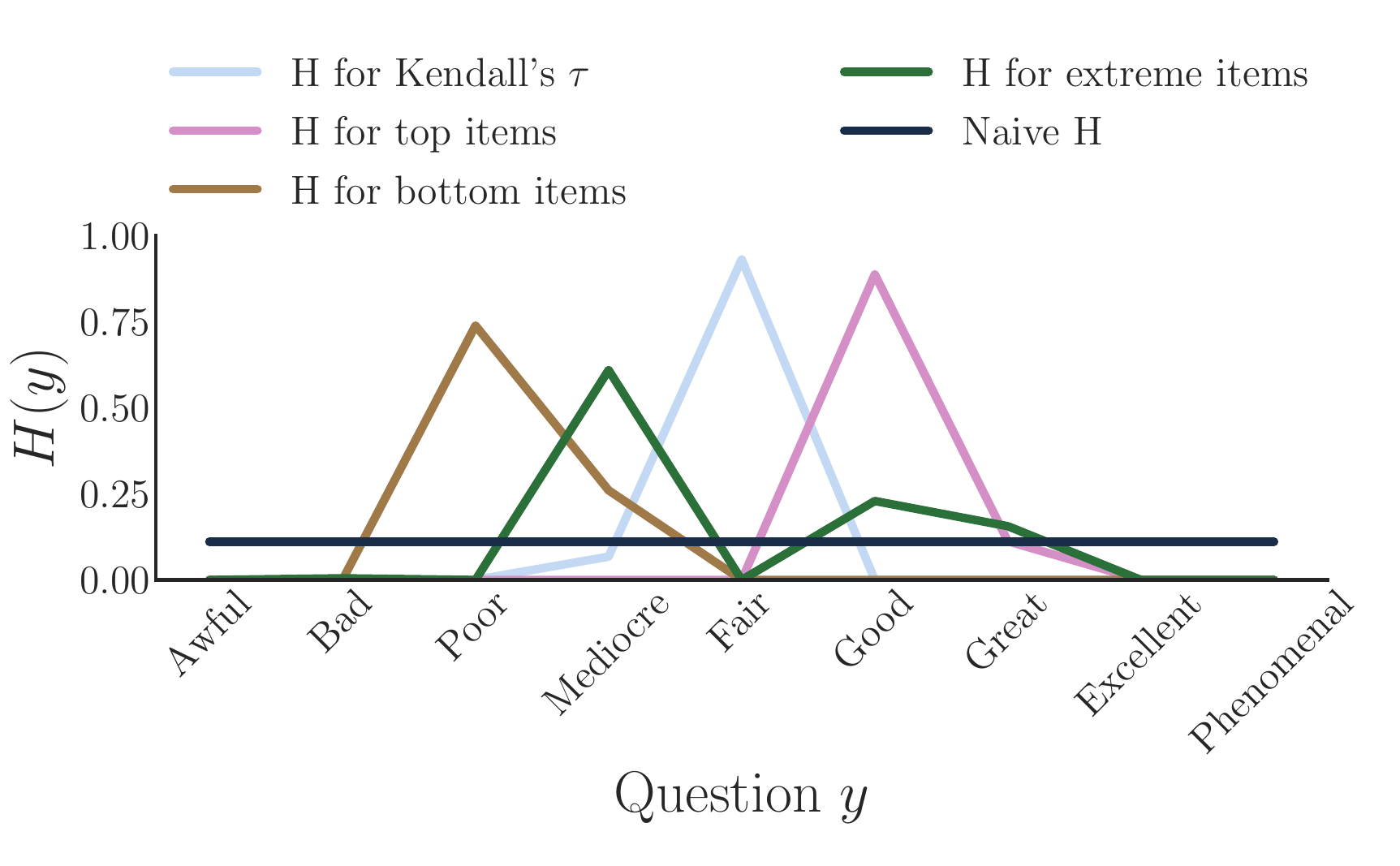}
		\caption{Uniform search, $g(\theta) = 1$.}
		\label{fig:stvaryw}
	\end{subfigure}\hfill
	\begin{subfigure}[t]{.48\linewidth}
		\includegraphics[width=\linewidth]{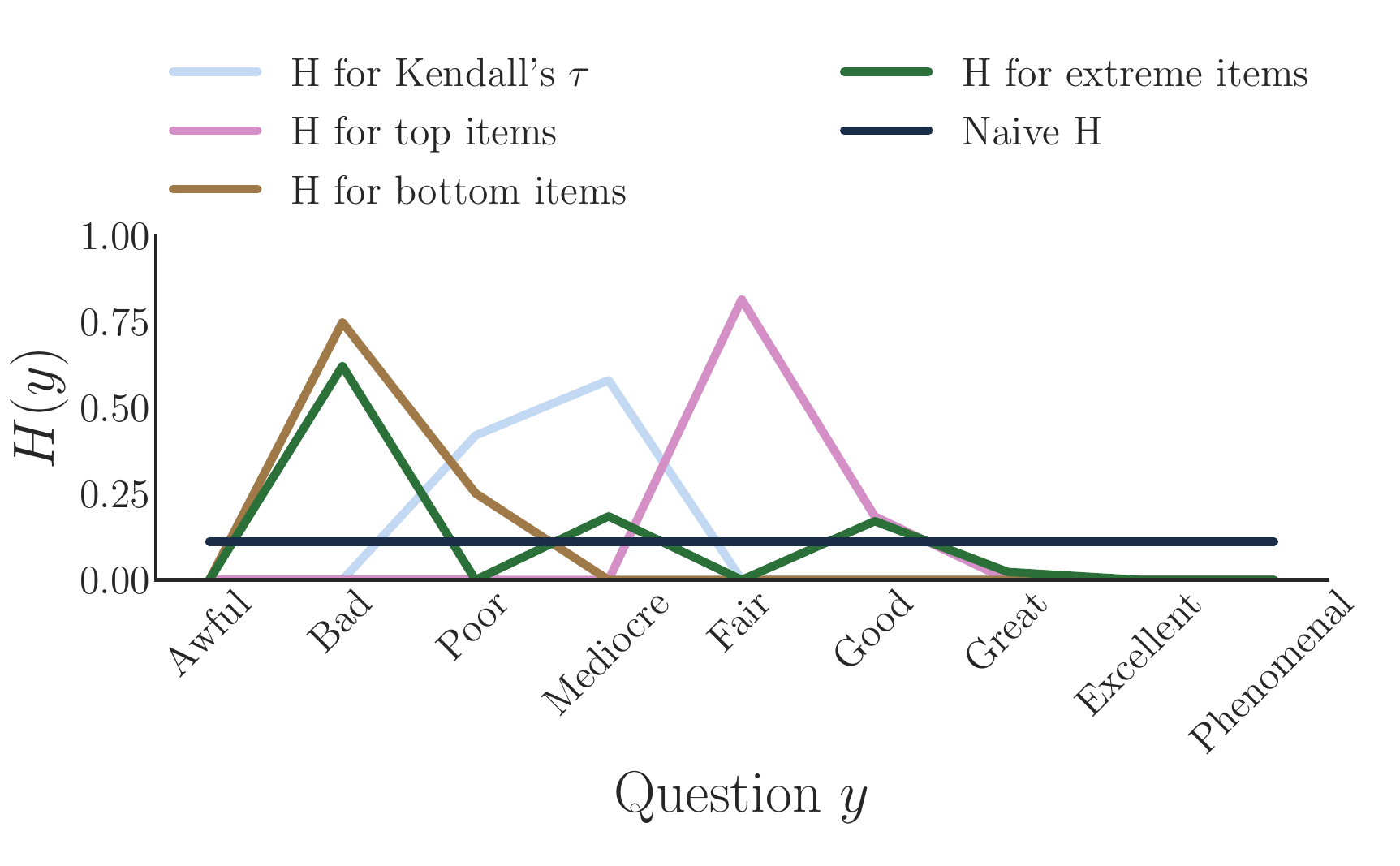}
		\caption{Linearly increasing search, $g(\theta) = \frac{1 + 10\theta}{11}$.}
		\label{fig:stlinearsearch}
	\end{subfigure}\hfill	
	\caption{Optimal $H(y)$ varying by $w(\theta_1, \theta_2)$ using Mechanical Turk data}
	\label{fig:mturkresults}
\end{figure}

Figure~\ref{fig:mturkresults} shows the optimal $H$ calculated for various platform settings. These distributions illustrate how often certain binary questions should be asked as it depends on the matching rates and platform objective. For example, as Figure~\ref{fig:stvaryw} shows, when there is uniform matching and the platform cares about the entire ranking (i.e. has Kendall's $\tau$ or Spearman's $\rho$ objective), it should ask most buyers to answer the question, ``Would you rate this item as having `Fair' quality or better?''. 

Several qualitative insights can be drawn from the optimal $H$. Most importantly, note that the optimal designs vary significantly with the platform objective and matching rates. In other words, given the same empirical data $\hat{\psi}$, the platform's design changes substantially based on its goals and how skewed matches are on the platform. Further, note that the differences in $H$ follow from the differences in $\beta$ that are illustrated in Figure~\ref{fig:betadifferent}: when the platform wants to accurately rank the best items, the questions that distinguish amongst the best (e.g., ``Would you rate this item as having `Good' quality?'') are drawn more often.  

\subsection{Simulation description}
\label{sec:mturksimulationdesc}

\begin{figure}[t]
	\centering
	\begin{subfigure}[t]{.5\linewidth}
		\includegraphics[width=\linewidth]{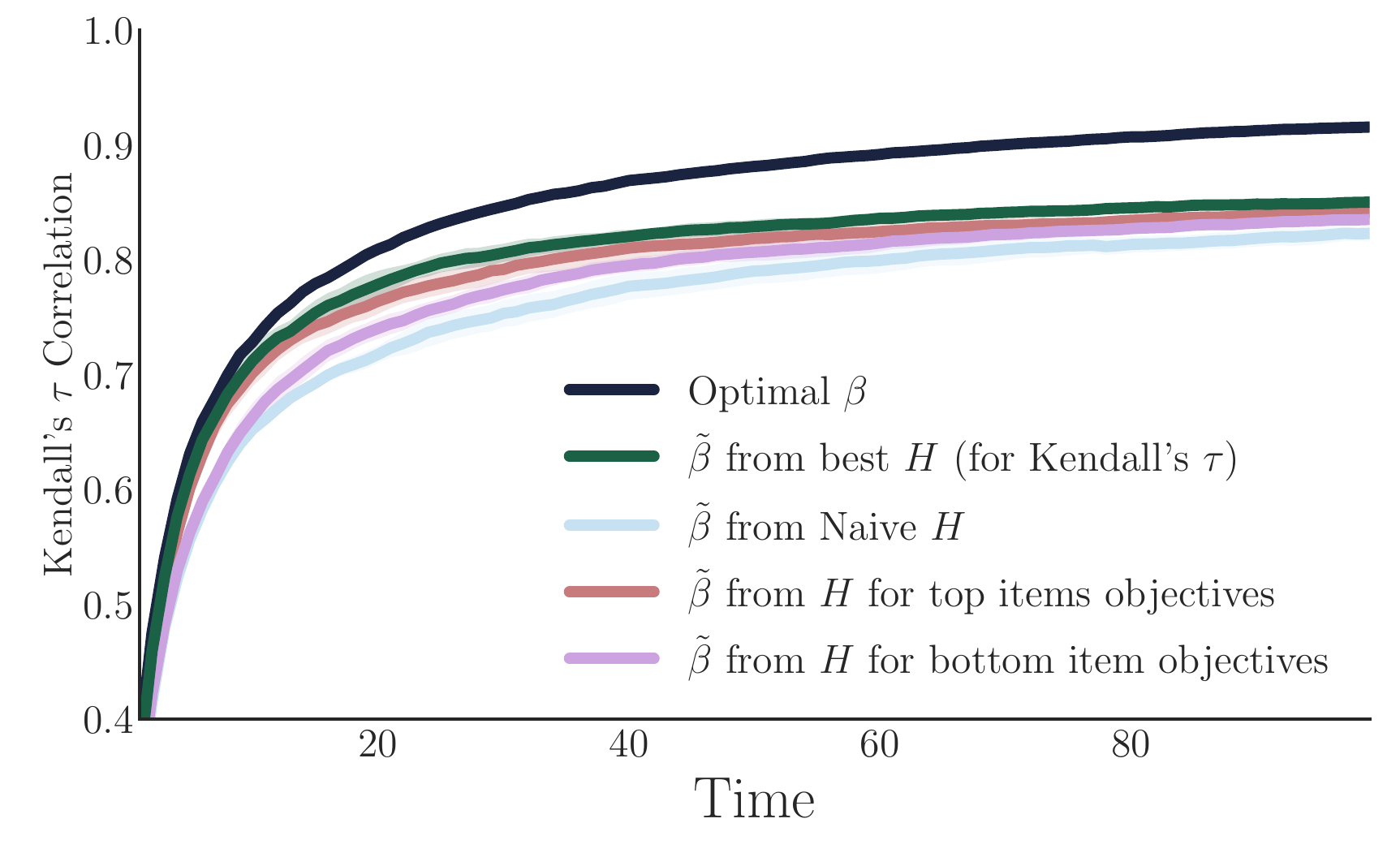}
		\caption{Uniform matching, no birth/death}
		\label{fig:mturk_unifdeath0}
	\end{subfigure}\hfill
	\begin{subfigure}[t]{.5\linewidth}
		\includegraphics[width=\linewidth]{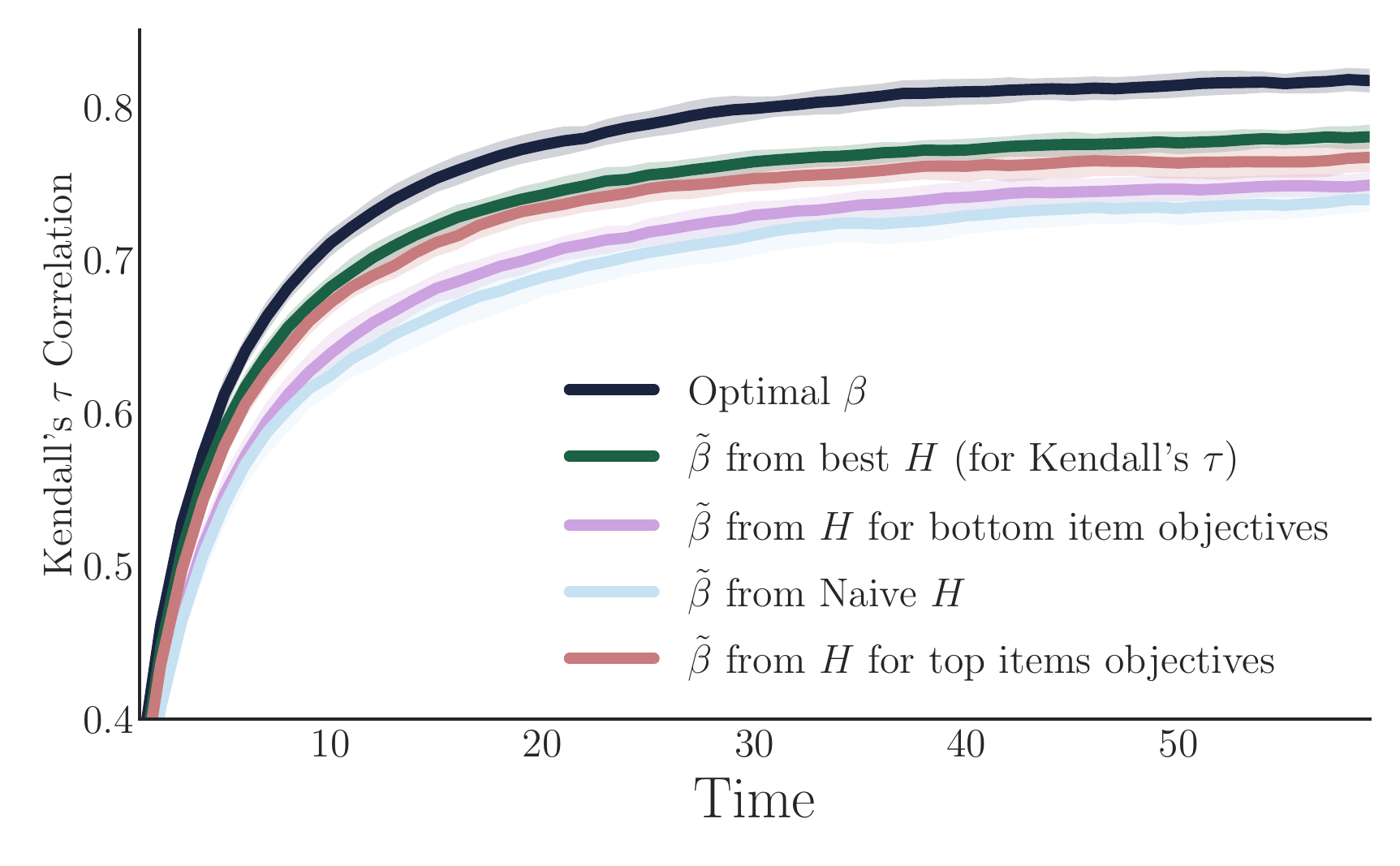}
		\caption{Uniform matching, 2\% probability of death per time}
		\label{fig:mturk_unifdeath02}
	\end{subfigure}\hfill
	
	\begin{subfigure}[t]{.5\linewidth}
	\includegraphics[width=\linewidth]{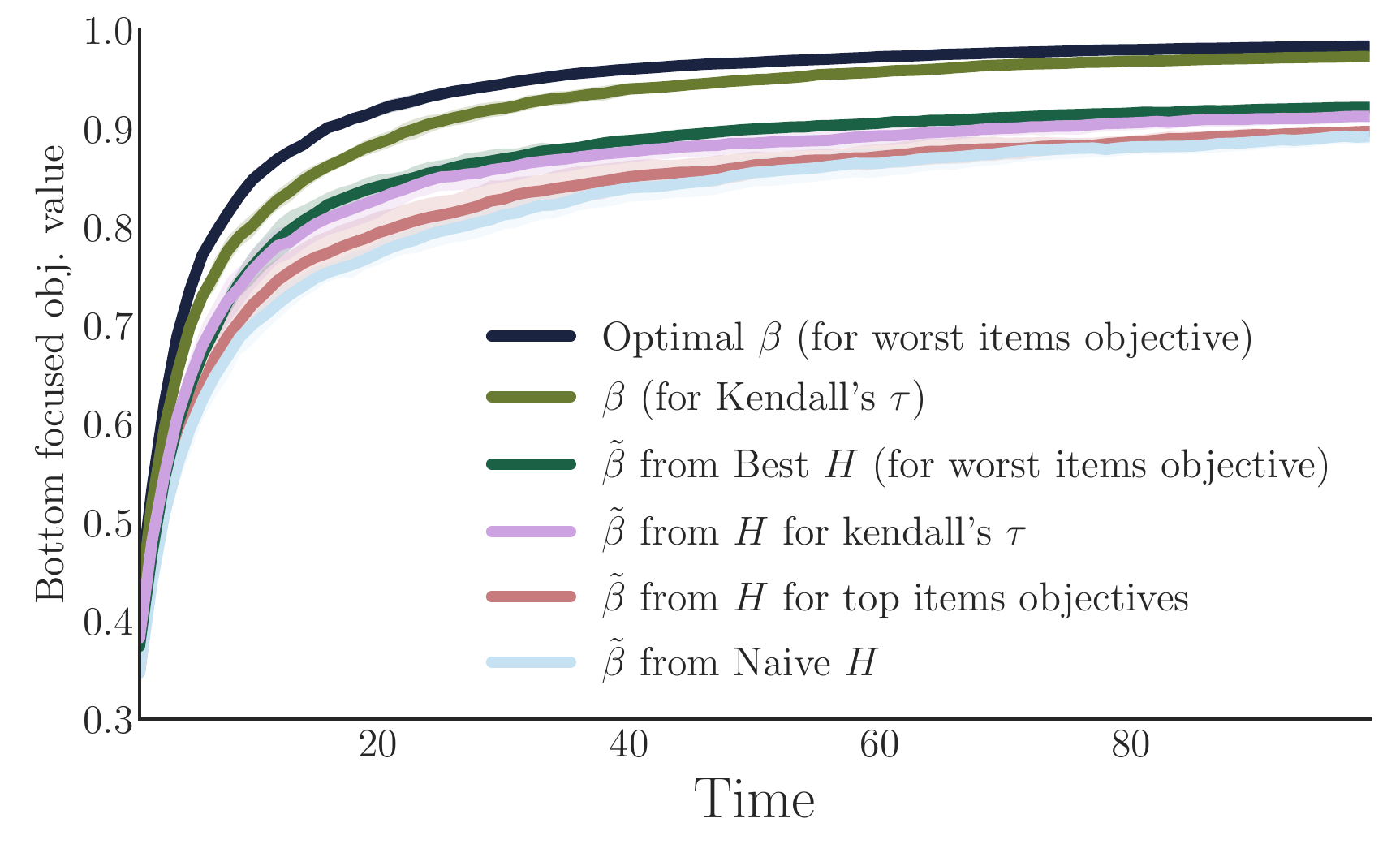}
	\caption{Uniform matching, no birth/death, worst items weighted objective}
	\label{fig:mturk_differentobjective}
	\end{subfigure}\hfill
	\begin{subfigure}[t]{.5\linewidth}
		\includegraphics[width=\linewidth]{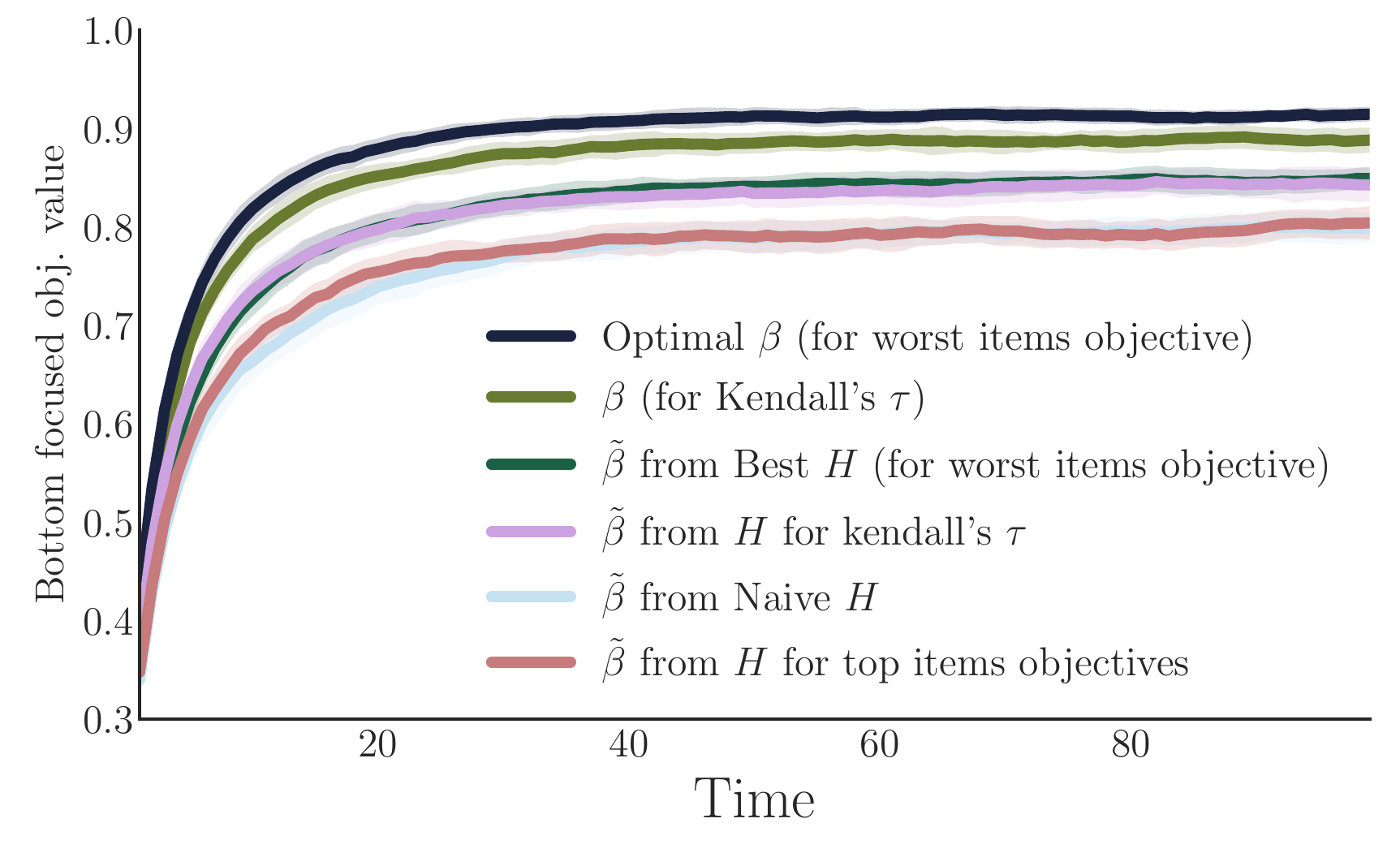}
		\caption{Uniform matching, 2$\%$ birth/death, worst items weighted objective}
		\label{fig:mturk_differentobjective_death}
	\end{subfigure}\hfill

	\begin{subfigure}[t]{.5\linewidth}
		\includegraphics[width=\linewidth]{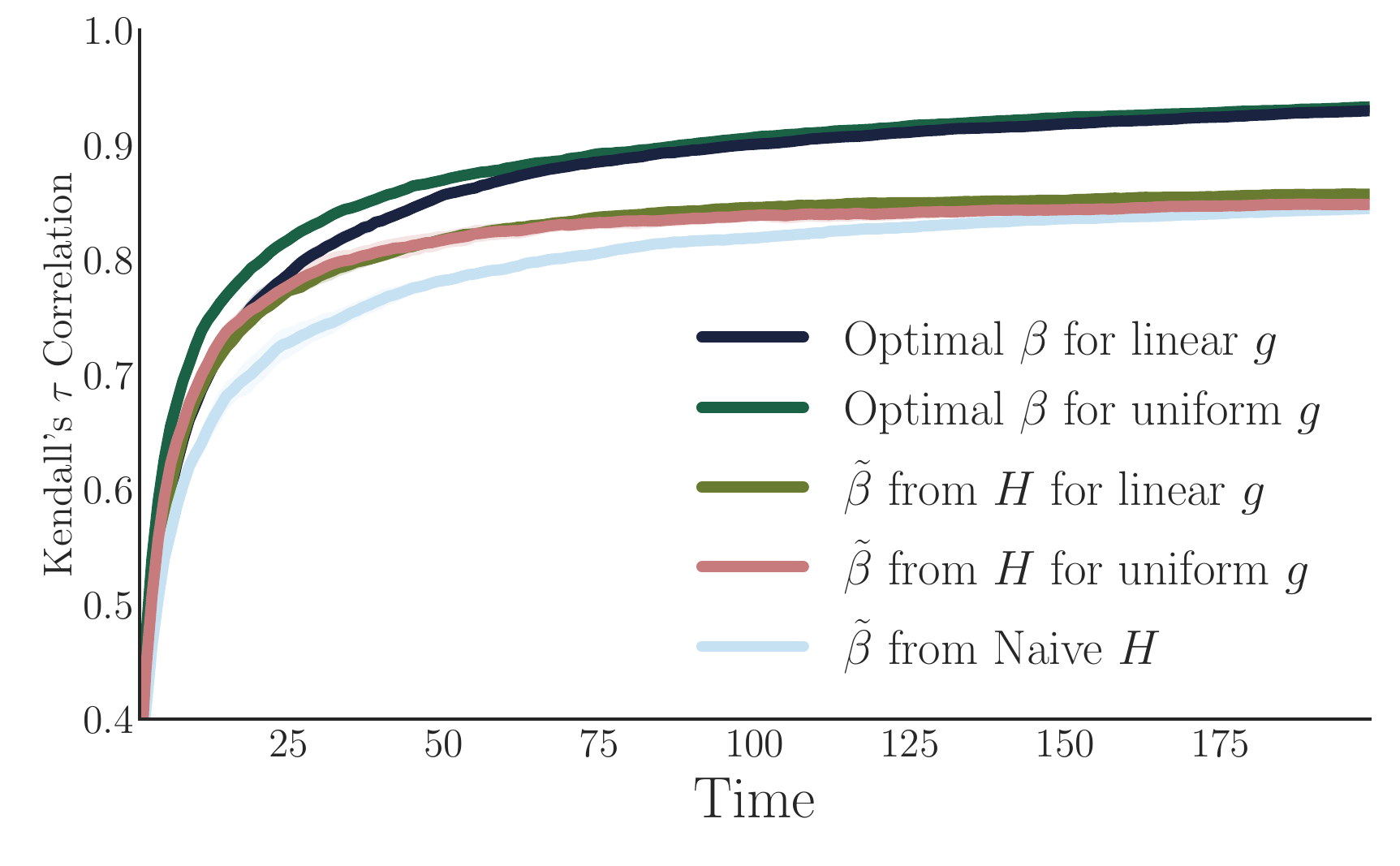}
		\caption{Linear matching, no birth/death}
		\label{fig:mturk_lineardeath0}
	\end{subfigure}\hfill
	\begin{subfigure}[t]{.5\linewidth}
		\includegraphics[width=\linewidth]{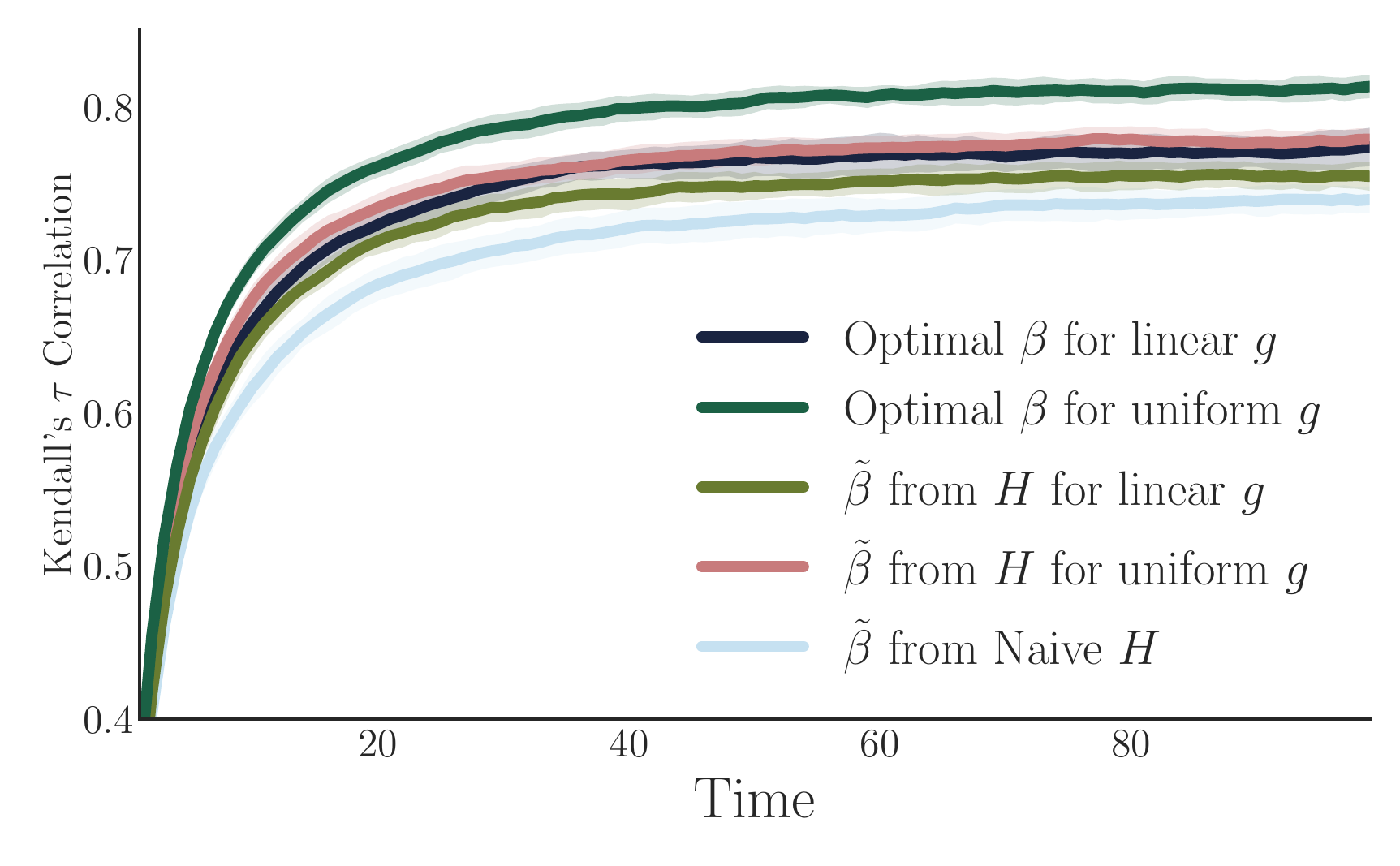}
		\caption{Linear matching, 2\% probability of death per time}
		\label{fig:mturk_lineardeath02}
	\end{subfigure}\hfill

	\caption{Simulations from data from Mechanical Turk experiment -- Binary rating system} 
	\label{fig:mturksimulationsbinary}
\end{figure}

Using the above data and subsequent designs, we simulate markets with a binary rating system as described in Section~\ref{sec:modelbinary}.  Our simulations have the following characteristics.
\begin{itemize}
	\item 500 items. Items have i.i.d. quality in $[0,1]$. For item with quality $\theta$,  we model buyer rating data using the $\psi$ collected from the experiment as follows. In particular, we presume the items are convex combinations of the representative items in our experiment -- items with quality $\theta \in [.1, .3]$ are assumed to have rating probabilities $\psi(\theta, y) = \alpha \psi(.1,y) + (1-\alpha)\psi(.3, y)$, where $\alpha = (\theta-.1)/.2$. Similarly for $\theta$ in other intervals. This process yields the $\tilde{\beta}$ shown in Figure~\ref{fig:betavsbetatilde}. 
	\item In some simulations, all items enter the market at time $k = 0$ and do not leave. In the others, with \textit{entry and exit}, each item independently leaves the market with probability $.02$ at the end of each time period, and a new item with quality drawn i.i.d. from $[0,1]$ enters.
	\item There are 100 buyers, each of which matches to an item independently. In other words, matching is independent across items, and items can match more than once per time period. 
	\item Matching is random with probability as a function of an item's \textit{estimated} rank $\hat \theta$ according to score, rather than actual rank. In other words, the optimal systems were designed assuming item $\theta$ would match at rate $g(\theta)$; instead it matches according to $g(\hat{\theta})$, where $\hat{\theta}$ is the item's rank according to score. We use both $g = 1$ and \textit{linear} search, $g(\hat \theta)  = \frac{1 + 10\hat{\theta}}{11}$. 
	\item $\mY$ is the set of 9 adjectives from our MTurk experiments. 
	\item We test several possible $H$: \textit{naive} with $H(y) = \frac{1}{|\mY|}$, and then the various optimal $H$ calculated for the different sections, illustrated in Figure~\ref{fig:mturkresults}. 
\end{itemize}

\subsubsection{Simulation results}
\label{sec:appmturkperformancesimulate}

Figure~\ref{fig:mturksimulationsbinary} contains plots from a simulated system that has binary ratings. Figures~\ref{fig:mturk_unifdeath0},~\ref{fig:mturk_unifdeath02} are with uniform search ($g = 1$),  Figures~\ref{fig:mturk_differentobjective} and \ref{fig:mturk_differentobjective_death} plot the objective prioritizing the worst items,
 and Figures~\ref{fig:mturk_lineardeath0} and \ref{fig:mturk_lineardeath02} are with linearly increasing search. For each setting, we include both plots with and without birth/death. 
 
 Together, the results suggest that the asymptotic and rate-wise optimality of our calculated $\beta$ hold even under deviations of the model, and that the real-world design approach outlined in Section~\ref{sec:applic_insights} would provide substantial information benefits to platforms. 
 
 Several specific qualitative insights can be drawn from the figures, alongside those discussed in the main text.
 \begin{enumerate}[leftmargin=*]
 	\item From all the plots with uniform search, the $H$ designed using our methods for the given setting outperforms other $H$ designs, as expected, and the optimal $\beta$ (for the given setting) significantly outperforms other designs both \textit{asymptotically} and \textit{rate-wise}.
 	\item Qualitatively, again with uniform search, heterogeneous item age also does not affect the results. In fact, it seems as if the optimal $\beta$ and best possible $H$ (given the data) as calculated from our methods outperforms other designs both \textit{asymptotically} and \textit{rate-wise}. Note that this is true even though items entering and leaving the market means that the system may not enter the asymptotics under which our theoretical results hold. 
	\item Figures~\ref{fig:mturk_differentobjective} and \ref{fig:mturk_differentobjective_death} show the same system parameters as  Figures~\ref{fig:mturk_unifdeath0},~\ref{fig:mturk_unifdeath02}, i.e. uniform search. However, while~\ref{fig:mturk_unifdeath0},~\ref{fig:mturk_unifdeath02} show Kendall's $\tau$ correlation over time,~\ref{fig:mturk_differentobjective} and \ref{fig:mturk_differentobjective_death} show the objective prioritizing bottom items ($w = (1-\theta_1)(1-\theta_2)(\theta_1-\theta_2)$). Note that the $\beta$ calculated for the actual objective outperforms that calculated for Kendall's $\tau$, including asymptotically.
	
	 Similarly, complementing the fact that $H$ design changes significantly with the weight function, these plots show the value of designing while taking into account one's true objective value -- the different designs perform differently. Mis-specifying one's objective (e.g. designing to differentiate the best items when one truly cares about the worst items) leads to a large gap in performance (e.g. see the gap between the dark green and red lines in \ref{fig:mturk_differentobjective} and \ref{fig:mturk_differentobjective_death}). 
	 
	 Note that comparing the performance of $\beta$ for the misspecified objective and $H$ for the true objective is not a fair comparison: the former differentiates between all items (though potentially not in a rate-optimal way), while $H$ is constrained by reality, i.e. $\psi$ and $\mY$. 

 	\item Now, consider Figures~\ref{fig:mturk_lineardeath0} and \ref{fig:mturk_lineardeath02}, which plot the system with linearly increasing search. Note that, contrary to expectation, the optimal $\beta$ for \textit{uniform} search outperforms the $\beta$ for the actual system simulated, with linear search! This pattern is especially true for small time $k$ and with item birth/death.  
 	
 	This inversion can be explained as follows. \textit{Uniformization} occurs with heterogeneous age and matching according to observed quality: new items of high type are likely to be mis-ranked lower, while new items of low type are more likely to be mis-ranked higher. (We note that this may not matter in practice, where the search function itself is fit through data, which already captures this effect.) These errors are prominent at low time $k$ and with item birth/death, i.e., in the latter our system never reaches the asymptotics at which the linear $\beta$ is the optimal design. 
 	
 	This pattern can be seen more clearly by comparing the two $\beta$ curves in Figures~\ref{fig:mturk_lineardeath0}, without item birth/death. At small $k$, when  errors are common and so search is more effectively uniform, the $\beta$ for uniform matching performs the best. However, as such errors subside over time, the performance of the $\beta$ for linear search catches up and eventually surpasses that of uniform optimal $\beta$. 
  
 \end{enumerate}
 
\FloatBarrier

\section{Supplementary theoretical information and results}

We now give some additional detail and develop additional results. Section~\ref{sec:appformalspec} contains the formal specification and update of our deterministic dynamical system. Section~\ref{sec:algorithm} gives our algorithm, Nested Bisection, is far more detailed pseudo-code. Section~\ref{sec:objectivematchingeffect} formalizes our earlier qualitative discussion on how matching rates affects the function $\beta$. Section~\ref{sec:discretization} includes a convergence result for functions $\beta_M$ as $M$ increases. Finally, Section~\ref{sec:realworlddesignexperiment} contains simple results on how one can learn $\psi(\theta, y)$ through experiments, even if one does not have a reference set of items $\Theta$ with known quality before one begins experiments. 

\subsection{Formal specification of system state update}
\label{sec:appformalspec}
Recall that $\mu_k(\Theta, X)$ is the mass of items with true quality $\theta \in \Theta \subseteq [0,1]$ and a reputation score $x\in X\subseteq [0,1]$ at time $k$. Let $E_k = \{ \theta : n_k(\theta) = n_{k-1}(\theta) + 1 \}$.  These are the items who receive an additional rating at time $k$; for all $\theta \in E_k^c$, $n_k(\theta) = n_{k-1}(\theta)$. Our system is completely deterministic, and evolves according to the distributions of the individual seller dynamics. 

For each $\theta\in E_k, x, x'$, define $\omega(\theta, x, x')$ as follows:
\[ \omega(\theta, x, x') = \beta(\theta) \bbI\{ n_k(\theta) x - n_{k-1}(\theta) x' = 1 \} + (1 - \beta(\theta)) \bbI \{ n_k(\theta) x - n_{k-1}(\theta) x' = 0 \}. \]
Then $\omega$ gives the probability of transition from $x'$ to $x$ when an item receives a rating.
We then have:
\[ \mu_{k+1}(\Theta, X) = \int_{E_k} \int_{x' = 0}^1 \int_{x\in X} \omega(\theta, x, x')  dx \mu_k(dx', d\theta) + \int_{E_k^c} \int_{x\in X} \mu_k(dx, d\theta). \]
It is straightforward but tedious to check that the preceding dynamics are well defined, given our primitives.

\subsection{Detailed algorithm}
\label{sec:algorithm}
Here, we present the Nested Bisection algorithm, which is described at a high level and summarized in pseudo-code in the main text, in more detail. 

\FloatBarrier
\begin{algorithm}
	\DontPrintSemicolon
	\KwData{Set size $M$, grid width $\delta$, match function $g$ \,\,\,\,\,\,\,\, \tcc{Assume $\delta << \min_i t_i - t_{i-1}$}}
	\KwResult{$\beta_M$ levels $\{t_0 \dots t_{M-1}\} $}
	\Fn{main ($M$, $\delta$, $g$)}{
		$t_0 = 0, t_{M-1} = 1$\;
		$\ell = 1 - \frac{1}{M-1}, u = 1 - \delta$ \;
		\While{$u - \ell > \delta/2$}{
			$j_{M-2} = \frac{r+\ell}{2}$\;
			$\text{rate}_\text{last} = -g_{M-2} \log (t_{M-2})$ \;
			$\{j_1 \dots  j_{M-3}\} = $ \textit{CalculateOtherLevels}($j_{M-2}, \text{rate}_\text{last}, g$)\;
			$\text{rate}_\text{first} = -g_1 \log (1-t_1)$\;
			\lIf{$\text{rate}_\text{first} < \text{rate}_\text{last}$}{
				$\ell = j_{M-2}$
			}
			\lElse{
				$u = j_{M-2}$
			}
		}
		$\{t_1 \dots  t_{M-2}\} = $ \textit{CalculateOtherLevels}($u, g$)\;
		$t_{M-2} = u$\;
		return $\{t_i\}$\;
	}
	\Fn{PairwiseRate ($t_{m-1}$, $t_{m}$, $g_{m}$, $g_{m-1}$)}{
		return $-(g_{{m-1}} + g_{m})\log \left[(1-t_{m-1})^{\frac{g_{m-1}}{g_{{m-1}} + g_{m}}}(1-t_{m})^{\frac{g_m}{g_{{m-1}} + g_{m}}} + t_{m-1}^{\frac{g_{m-1}}{g_{{m-1}} + g_{m}}}t_{m}^{\frac{g_m}{g_{{m-1}} + g_{m}}}\right]$
	}
	\Fn{CalculateOtherLevels ($j_{M-2}$, $\text{rate}_\text{target}$, $g$)}{
		\tcc{Given target rate from current guess $j_{M-2}$, sequentially fix other levels.}
		\ForEach{$m \in M-3 \dots 1$}{
			$j_m = $ \textit{BisectNextLevel}($j_{m+1}, \text{rate}_\text{target}, g_{m}, g_{m+1}$)\;
		}
		return $\{j_1 \dots j_{M-3}\}$\;
	}
	\Fn{BisectNextLevel ($j_m, \text{rate}_\text{target}, g_{m-1}, g_m$)}{
		$\ell = 0, r = j_m - \delta$ \;
		\While{$r - \ell > \delta/2$}{
			$j_{m-1} = \frac{r+\ell}{2}$\;
			$\text{rate}_\text{m} = $ \textit{PairwiseRate}($j_{m-1}$, $j_{m}$, $g_{m-1}$, $g_m$)\;
			\lIf{$\text{rate}_\text{m} \leq \text{rate}_\text{target}$}{		$r = j_{m-1}$		}
			\lElse{
				$\ell = j_{m-1}$
			}
		}
		return $r$\;
	}
	\caption{Nested Bisection given in more detail}
	\label{alg:Fasttimealgorithm_detailed}
\end{algorithm}
\FloatBarrier

\subsection{Formalization of effect of matching rates shifting}
\label{sec:objectivematchingeffect}

Matching concentrating at the top items moves mass of $\beta(\theta)$ away from high $\theta$, and subsequently mass of $H(y)$ \textit{away} from the questions that help distinguish the top items, as observed in Figures~\ref{fig:betadifferentmatching}~and~\ref{fig:stlinearsearch} above. Informally, this occurs because when matching concentrates, top items are accumulating many ratings more ratings comparatively, and so the amount of information needed per rating is comparatively less. We formalize this intuition in Lemma~\ref{lem:shiftgmass} below.

The lemma states that if matching rates shift such that there is an index $k$ above which matching rates increase and below which they decrease, then correspondingly the levels of $\beta$, (i.e. $t_i$) become closer together above $k$.

\begin{restatable}{lemma}{lemshiftgmass}
	Suppose $k, g, \tilde{g}$ such that $\forall j \in \{k+1 \dots M-1\}, g_j \geq \tilde{g}_j$, and $\forall j \in \{0 \dots k-1\}, g_j \leq \tilde{g}_j$, and $g_k =\tilde{g}_k$. Then, $t^*_k \geq \tilde{t}^*_k$. \label{lem:shiftgmass}
\end{restatable}

\begin{proof}
	This proof is similar to that of Lemma~\ref{lem:systemofequations}, except that with the matching function changing the overall rate function can either increase, decrease, or stay the same.	
	Suppose the overall rate function decreased or stayed the same when the matching function changed from $\tilde{g}$ to $g$. Then $g_{M-2} > \tilde{g}_{M-2}$ and the target rate is no larger, and so $t^*_{M-2} > \tilde{t}^*_{M-2}$. Then, $t^*_{M-3} > \tilde{t}^*_{M-3}$ (a smaller width is needed because the matching rates are higher and the rate is no larger, and the next value also increased). This shifting continues until $t^*_{k+1} > \tilde{t}^*_{k+1}$. Then, $t^*_{k} > \tilde{t}^*_{k}$. 
	
	Suppose instead that the overall rate function increased when the matching function changed from $\tilde{g}$ to $g$. Then $g_{1} < \tilde{g}_{1}$ and the target rate is larger, and so $t^*_{1} > \tilde{t}^*_{1}$. Then, $t^*_{2} > \tilde{t}^*_{2}$ (a larger width is needed and the previous value also increased). This shifting continues until $t^*_{k-1} > \tilde{t}^*_{k-1}$. Then, $t^*_{k} > \tilde{t}^*_{k}$.  	
\end{proof}

\subsection{Limit of $\beta$ as $M\to \infty$}
\label{sec:discretization}

Let $\beta^w_M$ denote the optimal $\beta$ with $M$ intervals for weight function $w$, with intervals $\{{S_i^{wM}}\} = \{[s_i^{wM}, s_{i+1}^{wM})\}$ and levels $\v{t^{wM}}$. Let $q_{wM}(\theta) = i/M$ when $\theta\in[s^{wM}_i, s^{wM}_{i+1})$, i.e. the \textit{quantile} of interval item of type $\theta$ is in. 

Then, we have the following convergence result for $\beta_M$. 

\begin{restatable}{theorem}{thmunifconvergencecleanmain}
	Let $g$ be uniform. Suppose $w$ such that $q_{wM}$ converges uniformly. Then, $\forall C\in \bbN, \exists \beta^w$ s.t. $\beta^w_{C2^N+1} \to \beta^w$ uniformly as $N\to\infty$. \label{thm:unifconvergencecleanmain}
\end{restatable}

The proof is technical and is below. We leverage the fact that, for $g$ uniform, the levels of $\beta_{2M}$ can be analytically written as a function of the levels of $\beta_M$. We believe (numerically observe) that this theorem holds for the entire sequence as opposed to the each such subsequence, and for general matching functions $g$. However, our proof technique does not carry over, and the proof would leverage more global properties of the optimal $\beta_M$.

Furthermore, the condition on $w$ is light. For example, it holds for Kendall's $\tau$, Spearman's $\rho$, and all other examples mentioned in this work. 

This convergence result suggests that the choice of $M$ when calculating a asymptotic and rate optimal $\beta$ is not consequential. As $M$ increases, the limiting value of $W_k$ increases to $1$ (i.e. the asymptotic value increases), but the optimal rate decreases to $0$. As discussed above, with strictly increasing and continuous $\beta$, the asymptotic value is $1$ but the large deviations rate does not exist, i.e. convergence is polynomial. 

This result could potentially be strengthened as follows: first, show convergence on the entire sequence as opposed to these exponential subsequences, as conjectured; second, show desirable properties of the limiting function itself. It is conceivable but not necessarily true that the limiting function is ``better'' than other strictly continuous increasing functions in some rate sense, even though the comparison through large deviations rate is degenerate.

\subsection{Learning $\psi(\theta, y)$ through experiments}
\label{sec:realworlddesignexperiment}
Now, we show how a platform would run an experiment to decide to learn $\psi(\theta, y)$. In particular, one potential issue is that the platform does not have any items with know quality that it can use as representative items in its optimization. In this case, we show that it can use ratings within the experiment itself to identify these representative items. The results essentially follow from the law of large numbers.

We assume that $|\Theta| = L$ representative items $i\in\{1 \dots L\}$ are in the experiment, and each are matched $N$ times. The experiment proceeds as follows: every time an item is matched, show the buyer a random question from $\mY$. For each word $y\in \mY$, track the empirical $\hat \psi(i, y)$, the proportion of times a positive response was given to question $y$. Alternatively, if $\mY$ is totally ordered (i.e. a positive rating for a given $y$ also implies positive ratings would be given to all ``easier'' $y'$),  and can be phrased as a multiple choice question, data collection can be faster: e.g., as we do in our experiments: $\mY$ consists of a set of totally ordered adjectives that can describe the item; the rater is asked to pick an adjective out of the set; this is interpreted as the item receiving a positive answer to the questions induced by the chosen answer and all worse adjectives, and a negative answer to all better adjectives.

First, suppose the platform approximately knows the quality $\theta_i$ of each item $i$, and $\theta_i$ are evenly distributed in $[0,1]$. Suppose the items are ordered by index, i.e. $\theta_1 < \theta_2 <\dots<\theta_L$. Then let $\hat\psi(\theta, y) = \hat\psi(i, y)$ when $\theta\in[\theta_{i-1}, \theta_i]$. Call this procedure \textit{KnownTypeExperiment}.

\begin{lemma} Suppose ${\psi}(\theta, y)$ is Lipschitz continuous in $\theta$. With \textit{KnownTypeExperiment}, $\hat\psi(i, y) \to {\psi}(\theta_i, y) \forall y$ uniformly as $N \to \infty$. As $L\to\infty$, $\hat\psi(\theta, y) \to {\psi}(\theta, y) \forall \theta$ uniformly. \label{lem:knowntypeexperiment}
\end{lemma}

\begin{proof}
	The proof follows directly from the Strong Law of Large Numbers. As $N\to \infty$, $\forall i$, $\hat\psi(i, x) \to \psi(\theta_i, x)$ uniformly. Now, let $L\to\infty$. $\forall \epsilon, \exists L'$ s.t. $\forall L>L'$, $\forall \theta, \exists i$ s.t. $|\theta - \theta_i| < \epsilon$. $\psi(\theta, x)$ is Lipschitz in $\theta$ by assumption, and so $\hat{\psi}(\theta, x) \to \psi(\theta, x)$ uniformly. 
\end{proof}

We now relax the assumption that the platform has an existing set of items with known qualities. Suppose instead the platform has many items $L$ of unknown quality who are expected to match $N$ times each over the experiment time period. For each item, the platform would again ask questions from $\mY$, drawn according to any distribution (with positive mass on each question). Then generate $\hat\psi(\theta, y)$ as follows: first, rank the items according to their ratings during the experiment itself. Then, for each $y$, $\hat\psi(\theta,y)$ is the empirical performance of the $\theta$th percentile item in the ranking, i.e. $\hat\psi(\theta, y) = \hat\psi(\theta_i, y)$ for $\theta \in \left[\frac{i-1}{L}, \frac{i}{L}\right]$. Call this procedure \textit{UnknownTypeExperiment}.

\begin{lemma} Suppose ${\psi}(\theta, y)$ is Lipschitz continuous in $\theta$. With \textit{UnknownTypeExperiment}, $\hat\psi(\theta, y) \to {\psi}(\theta, y) \forall y,\theta$ uniformly as $L,N \to \infty$.  \label{lem:unknowntypeexperiment} 
\end{lemma}
\begin{proof}
	Fix $L$. Denote each item in the experiment as $i\in\{1 \dots L\}$ (with true quality $\theta_i \neq \theta_j$), and each item has $N$ samples. Without loss of generality, assume the items are indexed according to their rank on the average of their scores on the samples, defined as the percentage of positive ratings received. $i=1$ is then the worst item, and $i=L$ is the best item according to scores in the experiment.

	For $\psi(\theta, x)$ increasing in $\theta$, as $N\to\infty$, $Pr(\theta_i > \theta_j | i < j) \to 0$ almost surely by SLLN, and for a fixed $L$, $\{\theta_i\}$ this convergence is uniform. Furthermore, by SLLN, $\hat{\psi}(i, x) \to \psi(\theta_i, x)$ as $N\to \infty$. Recall $\hat{\psi}(\theta, x) = \hat{\psi}(i, x)$ for $\theta \in \left[\frac{i-1}{L}, \frac{i}{L}\right]$.  
	
	Now, let $L\to\infty$. $\forall \epsilon, \exists L'$ s.t. $\forall L>L'$, $\forall \theta, \exists i$ s.t. $|\theta - \theta_i| < \epsilon$. $\psi(\theta, x)$ is Lipschitz in $\theta$ by assumption, and so $\hat{\psi}(\theta, x) \to \psi(\theta, x)$ uniformly. 
\end{proof}

\section{Proofs}

In this Appendix section, we prove our results.

Sections~\ref{sec:pkratefunctions}-\ref{sec:wkratefunction} develop rate functions for $P_k$ and $W_k$. While rates for $P_k$ follow immediately from large deviation results, the rate function for $W_k$ requires more effort as the quantity is an integral over a continuum of $(\theta_1, \theta_2)$, each of which has a rate corresponding to that of $P_k(\theta_1, \theta_2)$. 

Then in Section~\ref{sec:provemainlemmas} we prove Theorem~\ref{thm:CanMaximizeAsymValueAndRate} and Lemma~\ref{lem:systemofequations}. 

Section~\ref{sec:additionallemmas} then contains additional necessary lemmas required for the proof of the algorithm and convergence result, Theorem~\ref{thm:unifconvergencecleanmain}. The main difficulty for the former is showing a Lipschitz constant in the resulting rate if a level $t_i$ is shifted, which requires lower and upper bounds for $t_1$ and $t_{M-2}$, respectively. For the former, we need to relate the solutions of the sequence of optimization problems used to find $\beta_M$ as $M$ increases. It turns out that both properties follow by relating the levels of $\beta_M$ to those of $\beta_{2M-1}$.

These additional lemmas are used to prove the algorithm approximation bound (Theorem~\ref{thm:algfinds}) and the convergence result (Theorem~\ref{thm:unifconvergencecleanmain}) in Section~\ref{sec:provealgorithm} and \ref{sec:proveconvergence}, respectively. 

Finally in Section~\ref{sec:appktsr} we prove the comments we make in the main text about Kendall's $\tau$ and Spearman's $\rho$ rank correlations belonging in our class of objective functions, with asymptotic values of $W_k$ maximized when $\v{s}$ is equispaced in $[0,1]$.

\subsection{Rate functions for $P_k(\theta_1, \theta_2)$}
\label{sec:pkratefunctions}
\begin{lemma} 
	$$	\lim_{k \to \infty}- \frac{1}{k} \log \left[\mu((x_k(\theta_1) - x_k(\theta_2)) \leq 0 | \theta_1 , \theta_2)\right] = \inf_{a \in \bbR} \left \{ g(\theta_1) I(a|\theta_1) + g(\theta_2) I(a|\theta_2) \right \}$$ where $I(a|\ell) = \sup_z \{ za - \Lambda(z|\theta) \}$, and $\Lambda(z|\theta)$ is the log moment generating function of a single sample from $x(\theta_1)$ and $g(\theta)$ is the sampling rate. \label{lem:problessthan}
\end{lemma}
\begin{proof}
	\noindent$\lim_{k \to \infty}- \frac{1}{k} \log \left[\mu((x_k(\theta_1) - x_k(\theta_2)) \leq 0 | \theta_1 , \theta_2)\right]$
	\begin{align}
	&= \lim_{k \to \infty}- \frac{1}{k} \log \left[\int_{a\in \bbR} \mu((x_k(\theta_1) = a | \theta_1 ) \mu(x_k(\theta_2) \geq a |\theta_2)da\right]\\
	&= \lim_{k \to \infty}- \frac{1}{k} \log \left[\int_{a\in \bbR} e^{-kg(\theta_1) I (a | \theta_1)} e^{-kg(\theta_2) I (a | \theta_2)}da\right] \label{eqnpart:applyld}\\
	&= \inf_{a \in \bbR} \left \{ g(\theta_1) I(a|\theta_1) + g(\theta_2) I(a|\theta_2) \right \} & \text{Laplace principle}
	\end{align}
\end{proof}

Where~\eqref{eqnpart:applyld} is a basic result from large deviations, where $kg(\theta_i)$ is the number of samples item of quality $\theta_i$ has received.

Note that this lemma also appears in \citet{glynn_large_2004}, which uses the Gartner-Ellis Theorem in the proof. Our proof is conceptually similar but instead uses Laplace's principle.

Recall that $\text{KL}(a||b) = a \log\frac{b}{a} + (1 - a) \log \frac{1-b}{1-a}$ is the Kullback-Leibler (KL) divergence between Bernoulli random variables with success probabilities $a$ and $b$ respectively. 
It is well known that for a Bernoulli random variable with success probability $t$, \begin{equation*}I(a | t) = \text{KL}(a || t)\end{equation*}
Then, we have
\begin{restatable}{lemma}{lemPkld}
	\label{lem:Pk_LD}
	Let $\theta_1 > \theta_2$ and $I(a|\theta) = \text{KL}(a || \beta(\theta))$. Further, Let $\b P_k(\theta_1, \theta_2) = 1-P_k(\theta_1, \theta_2)$. Then, 
	\begin{equation}
	\label{eq:LDrate}
	- \lim_{k \to \infty} \frac{1}{k} \log \b{P}_k(\theta_1, \theta_2) = \inf_{a \in \bbR} \left \{ g(\theta_1) I(a|\theta_1) + g(\theta_2) I(a|\theta_2) \right \},
	\end{equation}
\end{restatable}
\begin{proof}
	Follows directly from Lemma \ref{lem:problessthan}. 
	\begin{align*}
	- \lim_{k \to \infty} &\frac{1}{k} \log \b{P}_k(\theta_1, \theta_2 | \beta) \\
	&= \lim_{k \to \infty}- \frac{1}{k} \log \left[1 + \mu_k(x_k(\theta_1) - x_k(\theta_2) < 0 | \theta_1 , \theta_2) - \mu_k(x_k(\theta_1) - x_k(\theta_2) > 0 | \theta_1 , \theta_2)\right] \\
	&= \lim_{k \to \infty}- \frac{1}{k} \log \left[2\mu_k(x_k(\theta_1) - x_k(\theta_2) < 0 | \theta_1 , \theta_2) + \mu_k(x_k(\theta_1) - x_k(\theta_2) = 0 | \theta_1 , \theta_2)\right] \\
	&= \lim_{k \to \infty}- \frac{1}{k} \log \left[\mu_k(x_k(\theta_1) - x_k(\theta_2) \leq 0 | \theta_1 , \theta_2)\right]\\
	&= \inf_{a \in \bbR} \left \{ g(\theta_1) I(a|\theta_1) + g(\theta_2) I(a|\theta_2) \right\} &\text{Lemma}~\ref{lem:problessthan}
	\end{align*}
\end{proof}

\subsection{Laplace's principle with sequence of rate functions}

In order to derive a rate function for $\b W_k = (\lim_k W_k) - W_k$, we need to be able to relate its rate to that of $\b P_k(\theta_1, \theta_2)$. The following theorem, related to Laplace's principle for large deviations allows us to do so. 

\begin{theorem}
Suppose that $X$ is compact with finite Lebesgue measure $\mu(X) < \infty$.  Suppose that $\varphi(x)$ has an essential infimum $\ul{\varphi}$ on $X$, that $\varphi_n(x)$ has an essential infimum $\ul{\varphi}_n$, that both $\varphi$ and all $\varphi_n$ are nonnegative, and that $\varphi_n \to \varphi$ uniformly:
\[ \lim_{n \to \infty} \sup_{x \in X} | \varphi_n(x) - \varphi(x) | = 0. \]

Then:
\begin{equation}
\label{eq:laplace}
 \lim_{n \to \infty} \frac{1}{n} \log \int_X e^{-n\varphi_n(x)} dx = - \ul{\varphi}.
\end{equation}
\label{thm:laplace}
\end{theorem}

{\em Proof}.  First, we note that for all $x$ and $n$, $e^{-n \varphi_n(x)} \leq e^{-n \ul{\varphi}_n}$.  Therefore, letting $(*)$ denote the LHS of \eqref{eq:laplace}, we have:
\[ (*) \leq \lim_{n \to \infty} \frac{1}{n} \log \int_X e^{-n\ul{\varphi}_n} dx = - \ul{\varphi}, \]
where the last limit follows from the fact that $\varphi_n$ converges uniformly to $\varphi$, so that $\ul{\varphi}_n \to \ul{\varphi}$.

Next, for $\epsilon > 0$  let $A_n(\epsilon) = \{ x : \varphi_n(x) \leq \ul{\varphi}_n + \epsilon \}$ and let $A(\epsilon) = \{ x : \varphi(x) \leq \ul{\varphi} + \epsilon \}$.  It follows (again by uniform convergence) that for all sufficiently large $n$, $A(\epsilon/2) \subseteq A_n(\epsilon)$, so that $\mu(A(\epsilon/2)) \leq \mu(A_n(\epsilon))$ for all sufficiently large $n$.  Further, $\mu(A(\epsilon/2)) > 0$, since $\ul\varphi$ is the essential infimum.

Since:
\[ \int_X e^{-n \varphi_n(x)} dx \geq \mu(A_n(\epsilon)) e^{-n (\ul{\varphi}_n  + \epsilon)}, \]
it follows that:
\[ (*) \geq -\ul{\varphi} - \epsilon + \lim_{n \to \infty} \frac{1}{n} \log \mu(A_n(\epsilon)). \]
To complete the proof, observe that since $\mu(A_n(\epsilon))$ is bounded below by a positive constant for all sufficiently large $n$, the last limit is zero.  Therefore:
\[ (*) \geq -\ul{\varphi} - \epsilon. \]
Since $\epsilon$ was arbitrary, this completes the proof.  \qed

\begin{remark}
	Let $X = [0,1]\times[0,1]$, $\varphi_n(\theta_1, \theta_2) = -\frac{1}{n}\log \b P_n(\theta_1, \theta_2)$. Then, all the conditions for Theorem~\ref{thm:laplace} are met.
\end{remark}

\subsection{Rate function for $W_k$}
\label{sec:wkratefunction}
Our next lemma shows that we can obtain a nontrivial large deviations rate for ${W}_k$ when $\beta$ is a step-wise increasing function. 

Recall $W_k = \int_{\theta_1 > \theta_2} w(\theta_1, \theta_2) P_k(\theta_1, \theta_2 | \beta) d(\theta_1,\theta_2)$.

Let $\b P_k(\theta_1, \theta_2) = 1-P_k(\theta_1, \theta_2)$.

Further, let $\b{W}_k = (\lim_k W_k) - {W}_k = \int_{\theta_1 > \theta_2} w(\theta_1, \theta_2) \b{P}_k(\theta_1, \theta_2 | \beta) d(\theta_1,\theta_2)$. (recall we assumed $w$ integrates to 1 without loss of generality).

\begin{restatable}{lemma}{lemhatWbin}
	\label{lem:hatW_bin}
	Suppose $\beta$ is piecewise constant with $M$ levels $\{t_i\}$. Let $g_i \triangleq \inf_{\theta\in S_i}g(\theta) = g(s_i)$ Then,
	\begin{equation}
	\label{eq:rW}
	- \lim_{k \to \infty} \frac{1}{k} \log \b{W}_k = \min_{0 \leq i \leq M-2} \inf_{a \in \bbR} \left \{ g_{i+1} I(a|t_{i+1}) + g_i I(a|t_i)\right\} \triangleq r,
	\end{equation}
	where $I(a|t)=\text{KL}(a || t)$ as defined in Lemma \ref{lem:Pk_LD}.
\end{restatable}

\begin{proof}
	When $\beta$ is step-wise increasing with $M$ levels $\{t_i\}$, then \begin{equation*}\b{W}_k = \sum_{0  \leq i < j < M} \int_{\theta_2\in S_i, \theta_1\in S_j } w(\theta_1, \theta_2) \b{P}_k(\theta_1, \theta_2 | \beta) d(\theta_1,\theta_2)\end{equation*} as $\b{P}_k(\theta_1, \theta_2) = 0$ when $\beta(\theta_1) = \beta(\theta_2)$. 
	\\$- \lim_{k \to \infty} \frac{1}{k} \log \b{W}_k$
	\begin{align}
	 &= - \lim_{k \to \infty} \frac{1}{k} \log \int_{\theta_1 > \theta_2} w(\theta_1, \theta_2) \b{P}_k(\theta_1, \theta_2 | \beta) d(\theta_1,\theta_2)\\
	&= - \lim_{k \to \infty} \frac{1}{k} \log \sum_{0  \leq i < j < M} \int_{\theta_2\in S_i, \theta_1\in S_j } w(\theta_1, \theta_2) \b{P}_k(\theta_1, \theta_2 | \beta) d(\theta_1,\theta_2)\\
		&= -\max_{0\leq i < j < M}\left( \lim_{k \to \infty} \frac{1}{k} \left[\log \int_{\theta_2\in S_i, \theta_1\in S_j } w(\theta_1, \theta_2)\b{P}_k(\theta_j, \theta_i | \beta)d(\theta_1,\theta_2) \right]\right)\label{eqnstep:ldpproperty}\\
	&= -\max_{0\leq i < j < M}\sup_{\theta_1\in S_j,\theta_2\in S_i}\left( \lim_{k \to \infty} \frac{1}{k} \left[\log w(\theta_1, \theta_2)\b{P}_k(\theta_j, \theta_i | \beta) \right]\right)\label{eqnstep:ldppropertyintegral}\\
		&= -\max_{0\leq i < j < M}\sup_{\theta_1\in S_j,\theta_2\in S_i}\left( \lim_{k \to \infty} \frac{1}{k} \log \b{P}_k(\theta_j, \theta_i | \beta) \right)\\				
		&= \min_{0\leq i < j < M}\inf_{\theta_1\in S_j,\theta_2\in S_i}\left( -\lim_{k \to \infty} \frac{1}{k} \log \b{P}_k(\theta_j, \theta_i | \beta) \right) \label{eq:lastlinewithoutstepwise}\\	
	&= \min_{0 \leq i < j < M}\inf_{\theta_1\in S_j,\theta_2\in S_i} \inf_{a \in \bbR} \left \{ g(\theta_1) I(a|t_j) + g(\theta_2) I(a|t_i)\right\}\\
	&= \min_{0 \leq i < j < M} \inf_{a \in \bbR} \left \{ g_j I(a|t_j) + g_i I(a|t_i)\right\}\\
	&= \min_{0 \leq i < M-1} \inf_{a \in \bbR} \left \{ g_{i+1} I(a|t_{i+1}) + g_i I(a|t_i)\right\}
	\end{align}
	The last line follows from adjacent $t_i, t_{i+1}$ dominating the rate due to monotonicity properties.
	
	Line~\eqref{eqnstep:ldppropertyintegral} follows from Theorem~\ref{thm:laplace}.
	
	Line~\eqref{eqnstep:ldpproperty} follows from: $\forall a^\epsilon_i \geq 0$, ${\lim \sup}_{\epsilon \to 0} \left[\epsilon\log\left(\sum_i^N a^\epsilon_i\right)\right] = \max^N_i {\lim \sup}_{\epsilon \to 0} \epsilon \log(a^\epsilon_i)$, which is a finite case version (with fewer assumptions) of Theorem~\ref{thm:laplace}. See, e.g., Lemma 1.2.15 in~\citep{dembo_large_2010} for a proof of this property.

\end{proof}

\begin{restatable}{lemma}{lemNoLargeDeviationsRate}
	$\beta(\theta)$ is piecewise constant $\iff$ $\exists c(\beta)>0$ s.t.  $-\lim_{k\to\infty} \frac{1}{k} \log (\b W_k) = c(\beta)$.

	\label{lem:nolargedeviationsrate}
\end{restatable}
\begin{proof}
	$\implies$ follows directly from Lemma~\ref{lem:hatW_bin}: $\inf_{a \in \bbR} \left \{ g_{i+1} I(a|t_{i+1}) + g_i I(a|t_i)\right\}>0$ when $t_i \neq t_{i+1}$, which holds when $\beta$ is piece-wise constant with the appropriate number of levels.
	\\\\
	$\impliedby$ Consider $\beta$ that is not piece-wise constant. Recall that we further assume that $\beta$ is non-decreasing, and discontinuous only on a measure $0$ set. Following algebra steps similar to those in Lemma~\ref{lem:hatW_bin}, but for general $\beta$:
	
	\begin{align}
	- \lim_{k \to \infty} \frac{1}{k} \log \b{W}_k &= - \lim_{k \to \infty} \frac{1}{k} \log \int_{\theta_1 > \theta_2} w(\theta_1, \theta_2) \b{P}_k(\theta_1, \theta_2 | \beta) d(\theta_1,\theta_2)\\
	&=  \inf_{\theta_1>\theta_2}\left( -\lim_{k \to \infty} \frac{1}{k} \log \b{P}_k(\theta_j, \theta_i | \beta) \right)\\	
	&= 0
	\end{align}
	Where the last line follows from $\beta$ continuous at some $\theta_1$, and so $\lim_{\theta_2 \to \theta_1}\b{P}_k(\theta_1, \theta_2| \beta) = 1$.
	
	Intuitively, what goes wrong with continuous $\beta$ is that $\b{P}_k(\theta_1, \theta_2 | \beta)$ does not converge uniformly: 
	\begin{align*}
	\forall \epsilon, k, \exists \theta_2\neq \theta_1 &\,\,\,\, \b{P}_k(\theta_1, \theta_2) > \epsilon
	\end{align*}
	i.e. close by items are very hard to distinguish from one another.
	Then, because the large deviations rate of $\b W_k$ is dominated by the worst rates under the integral, we don't get a positive rate.
	
\end{proof}

\subsection{Proofs of Lemma \ref{lem:systemofequations} and Theorem~\ref{thm:CanMaximizeAsymValueAndRate}}
\label{sec:provemainlemmas}

\begin{remark}
	The KL divergence for two Bernoulli random variables is continuous and strictly convex, with minima at $a=b$, when $a,b \in (0,1)$. Note that $\inf_a \{g_i\text{KL}(a || t_i) + g(i+1)\text{KL}(a || t_{i+1})\}$, for all feasible $g$, is also continuous and strictly convex in $t_i$, $t_{i+1}$, with minima at $t_i = t_{i+1}$.\label{rem:movearound} 
\end{remark}
One consequence of the above fact is that fixing either $t_i$ or $t_{i+1}$ and moving the other farther away monotonically increases KL, while moving it closer decreases KL. 


\subsubsection{Proof of Theorem~\ref{thm:CanMaximizeAsymValueAndRate}}
\label{sec:provecanmaximizevalueandrate}
\begin{proof}
We use the same notation as the proof for Lemma~\ref{lem:hatW_bin}. 
	
Part 1. 
\begin{align}
\lim_{k \to \infty} W_k &= \lim_{k \to \infty} \sum_{0  \leq i < j < M} \left[\int_{\theta_2\in S_i, \theta_1\in S_j } w(\theta_1, \theta_2)P_k(\theta_1, \theta_2 | \beta)d(\theta_1,\theta_2)\right]  \label{eq:wintinsidebrackets_lim}\\
&=  \sum_{0  \leq i < j < M} \int_{\theta_2\in S_i, \theta_1\in S_j } w(\theta_1, \theta_2)d(\theta_1,\theta_2) & \label{eq:limWk}
\end{align}
\eqref{eq:limWk} follows from bounded convergence and $P_k(\theta_1, \theta_2 | \beta)\to 1$ for $\theta_1 \in S_j, t_2 \notin S_j$. Thus choosing $\v{s}$ to maximize~\eqref{eq:limWk} maximizes the asymptotic value of $W_k$. 
\\\\
Part 2. Follows directly from Lemma~\ref{lem:hatW_bin}.
	\end{proof}

\subsubsection{Proof of Lemma \ref{lem:systemofequations}}
\begin{proof}
	Recall $r(\v{t}) \triangleq - \lim_{k \to \infty} \frac{1}{k} \log (W - {W}_k) = \min_{0 \leq i \leq M-2} \inf_{a \in \bbR} \left\{ g_{i+1} \text{KL}(a || t_{i+1}) + g_i \text{KL}(a||t_i)\right\}$.
	\\\noindent We show the following:	
	\\${r}(\mathbf{t}) = $
	\begin{align*}\min\Bigg(&\log (1-t_1)^{-g_{1}}, \\&\log \left[(1-t_{i-1})^{\frac{g_{i-1}}{g_{{i-1}} + g_{i}}}(1-t_{i})^{\frac{g_i}{g_{{i-1}} + g_{i}}} + t_{i-1}^{\frac{g_{i-1}}{g_{{i-1}} + g_{i}}}t_{i}^{\frac{g_i}{g_{{i-1}} + g_{i}}}\right]^{-g_{{i-1}} - g_{i}} \text{ for }{1<i<M-1},\\& \log (t_{M-2})^{-g_{M-2}} \Bigg)\end{align*}and $\v{t}^*$ maximizes  ${r}_w(\v{t})$ $\iff$ all the terms inside the minimization ${r}_w(\v{t}^*)$ are equal. Further, the optimal levels $\v{t}^*$ are unique. The result immediately follows, that $\{t_i\}$ is the unique solution that equalizes the rates inside the minimization, by noting that the optimal $r$ has $t_0 = 0, t_{M-1}$ = 1. 
	
	We first prove the alternative form for ${r}$. Note that $\{g_{i-1}\text{KL}(a || t_{i-1}) + g_i\text{KL}(a || t_{i})\}$ is convex in $a$, and so we can find an analytic form for the infinum over $a$. 
	
	Let $a_i = \arg\inf_{a\in[t_{i-1}, t_i]} \{g_{i-1}\text{KL}(a || t_{i-1}) + g_i\text{KL}(a || t_{i})\} $
	\begin{align*}
	&\implies \nabla_{a_i} \left[g_{i-1}\text{KL}(a_i || t_{i-1}) + g_i\text{KL}(a_i || t_{i})\right] = 0\\
	&\implies \nabla_{a_i}\left[g_{i-1}\left(a_i \log \frac{a_i}{t_{i-1}} + (1-a_i)\log\frac{1-a_i}{1-t_{i-1}}\right) + g_i\left(a_i \log \frac{a_i}{t_{i}} + (1-a_i)\log\frac{1-a_i}{1-t_{i}}\right)\right] = 0\\
	&\implies g_{i-1}\left(\log \frac{a_i}{t_{i-1}} - \log\frac{1-a_i}{1-t_{i-1}}\right) + g_i\left(\log \frac{a_i}{t_{i}} - \log\frac{1-a_i}{1-t_{i}} \right) = 0\\
	& \implies \log\left(\frac{a_i}{1-a_i}\right)^{g_{i-1} + g_i} = \log \left(\frac{t_{i-1}}{1-t_{i-1}}\right)^{g_{i-1}} + \log\left(\frac{t_{i}}{1-t_{i}}\right)^{g_i}\\
	&\implies \frac{a_i}{1-a_i} = \left[\left(\frac{t_{i-1}}{1-t_{i-1}}\right)^{g_{i-1}}\left(\frac{t_{i}}{1-t_{i}}\right)^{g_i}\right]^{\frac{1}{g_{i-1} + g_i}}\\
	&\implies a_i = \frac{c}{1+c}, \,\text{where } c = \left[\left(\frac{t_{i-1}}{1-t_{i-1}}\right)^{g_{i-1}}\left(\frac{t_{i}}{1-t_{i}}\right)^{g_i}\right]^{\frac{1}{g_{i-1} + g_i}}
	\end{align*}
	Then,
	\\\noindent$g_{i-1}\text{KL}(a_i || t_{i-1}) + g_i\text{KL}(a_i || t_{i})$ 
	\begin{align}
	&= g_{{i-1}}a\log\frac{a}{t_{i-1}} + g_ia\log\frac{a}{t_{i}} + g_{{i-1}}(1-a)\log\frac{1-a}{1-t_{i-1}} + g_i(1-a)\log\frac{1-a}{1-t_{i}}\nonumber\\
	&= a\left[(g_{{i-1}} + g_{i}) \log\frac{a}{1-a}+ g_{i-1}\log\frac{1-t_{i-1}}{t_{i-1}} + g_i\log\frac{1-t_{i}}{t_{i}}\right] + \log (1-a)^{g_{{i-1}} + g_{i}}- \log(1-t_{i-1})^{g_{i-1}}(1-t_{i})^{g_i }\nonumber\\
	&= (g_{{i-1}} + g_{i}) \log (1-a) - \log(1-t_{i-1})^{g_{i-1}}(1-t_{i})^{g_i }\label{eqnpart:thingscancel}\\
	&= -(g_{{i-1}} + g_{i})\log \left[\left[1 + \left[\left(\frac{t_{i-1}}{1-t_{i-1}}\right)^{g_{i-1}}\left(\frac{t_{i}}{1-t_{i}}\right)^{g_i}\right]^{\frac{1}{g_{i-1} + g_i}}\right] (1-t_{i-1})^{\frac{g_{i-1}}{g_{{i-1}} + g_{i}}}(1-t_{i})^{\frac{g_i}{g_{{i-1}} + g_{i}}}\right]\nonumber\\
	&= -(g_{{i-1}} + g_{i})\log \left[(1-t_{i-1})^{\frac{g_{i-1}}{g_{{i-1}} + g_{i}}}(1-t_{i})^{\frac{g_i}{g_{{i-1}} + g_{i}}} + t_{i-1}^{\frac{g_{i-1}}{g_{{i-1}} + g_{i}}}t_{i}^{\frac{g_i}{g_{{i-1}} + g_{i}}}\right] \label{eqnpart:finalsolutionforminimization}
	\end{align}
	
	Where line~\eqref{eqnpart:thingscancel} uses $\frac{a}{1-a} = c$ and $(g_{{i-1}} + g_{i}) \log c = \log \left[\left(\frac{t_{i-1}}{1-t_{i-1}}\right)^{g_{i-1}}\left(\frac{t_{i}}{1-t_{i}}\right)^{g_i}\right] $. Note that the first and last rates emerge, respectively, by plugging in $t_0 = 0, t_M = 1$, which holds trivially at the optimum from monotonicity.
	
	We note that a similar derivation, of the large deviation rate for two binomial distributions with different probability of successes and match rates, appears in~\citet{glynn_large_2004}. In that work, the authors seek to optimize the $g$ in order to identify the single best item out of a set of possible items, and a concave program emerges. In this work, because we optimize the probability of successes and care about retrieving a ranking of the items, no such concave or convex program emerges. 
	\\\\\noindent Now we show that $\v{t}^*$ maximizes  ${r}_w(\v{t})$ $\iff$ all the terms inside the minimization ${r}_w(\v{t})$ are equal.
	
	equalizes $\implies$ optimal. Let $r(i)$ be the $i$th term in the minimization, starting at $i=1$. Note that (holding the other fixed) increasing $t_i$ increases the $i$th term monotonically and decreases the $(i+1)$th term monotonically. Suppose $\beta$ s.t. $r(i) = r(j) \forall i,j$. To increase the minimization term, one must increase $r(i)\,, \forall i$. To increase $r(1)$, $t_1$ must increase, regardless of what the other levels are. Then, to increase $r(2)$, $t_2$ must increase $\dots$ to increase $r(M-2)$, $t_{M-2}$ must increase. However, to increase $r(M-1)$, $t_{M-2}$ must decrease, and we have a contradiction. Thus, one cannot increase all terms simultaneously. 
	\\\\\noindent equalizes $\impliedby$ optimal. Suppose $\mathbf{t}$ maximizes $r(\mathbf{t})$ but the terms inside the minimization are not equal. Then $\exists i$ s.t. $r(i) = \min_j r(j)$ and either $r(i)\neq r(i-1)$ or $r(i)\neq r(i+1)$. $r(i)$ can be increased without lowering the overall rate. This method can be repeated $\forall i: r(i) = \min_j r(j)$ and so $\mathbf{t}$ would not be optimal, a contradiction.
	\\\\\noindent Uniqueness follows from the overall rate unique determining $t_1, t_{M-2}$ and so iteratively uniquely determining the rest.
\end{proof}

\subsection{Additional necessary lemmas}
\label{sec:additionallemmas}

\textit{Now, we begin the set-up that will lead to a proof for Theorem~\ref{thm:algfinds}. It turns out that proving the theorem requires, in the process, essentially proving our convergence result with $M\to\infty$, Theorem~\ref{thm:unifconvergencecleanmain}. For Theorem~\ref{thm:algfinds}, we need a lower bound for $t_1$ as a function of $M$. This seems hard to do in general. Luckily, in our case, there is a property for how $\v{t^*}$ changes when $M$ is doubled. Using this property, we can derive that ${t^*_1} \geq \mathcal{O}(M^{-3})$. }
\\\\

Recall that step-wise increasing $\beta$ with $M$ intervals $S_i = [s_i, s_{i+1})$ has levels $\{t_i\}_{i=0}^{M-1}$, where $t_0 = 1, t_{M-1} = 1$, and $s_0\triangleq 0$, $s_{M} \triangleq 1$. 

Furthermore, we use the following notation for the large deviation rate \begin{equation}
r_i = -(g_{{i-1}} + g_{i})\log \left[(1-t_{i-1})^{\frac{g_{i-1}}{g_{{i-1}} + g_{i}}}(1-t_{i})^{\frac{g_{i}}{g_{{i-1}} + g_{i}}} + t_{i-1}^{\frac{g_{i-1}}{g_{{i-1}} + g_{i}}}t_{i}^{\frac{g_{i}}{g_{{i-1}} + g_{i}}}\right]\end{equation} for $i \in \{1 \dots {M-1}\}$, which implies $r_1 = -g_1\log(1 - t_1)$ and $r_{{M-1}} = -g_{M-2}\log(t_{M-2})$. 

We further use $r^{M-1}$ to be the rate achieved by the optimal $\beta_{M}$ with $M$ intervals. 
 
\begin{lemma}
	Suppose $g$ uniform, i.e. $g_i = 1, \forall i$ and that $\beta_{M}$ has values $\{t_i\}_{i=0}^{{M-1}}$. Then $\beta_{2{M-1}}$ has values $\{t_i'\}_{i=0}^{2{M-2}}$, where $t_{2i}' = t_i, \forall i\in\{0 \dots {M-1} \}$, $t'_1 = \frac{1}{2}\left(1 - \sqrt{1- t_1}\right)$ and $t'_{2{M}-3} = \frac{1}{2}\left(1 + \sqrt{t_{M-2}}\right)$. \label{lem:beta2mfromm}
\end{lemma}

\begin{proof}
	We first set the values $t_{2i}' = t_i$ and then optimally choose the remaining values $t_k', k$ odd. Then, we show that the resulting large deviation rates between all adjacent pairs are equal. Then, by the proof of Lemma~\ref{lem:systemofequations}, which showed that equalizing the rates between adjacent intervals is a sufficient condition for optimality, $\beta_{2{M-1}}$ has the levels $\{t_i'\}_{i=0}^{2{M-2}}$. 
	
	Let $r'$ denote rates between adjacent $t'$ as $r$ does for $t$. Supposing $t'_2 = t_1$, we find $t'_1$ such that $r'_1 = r'_2$ and $t'_1 < t'_2$. 
	\begin{align*}
	-\log (1 - t'_1) &= -2 \log \left[\sqrt{(1-t'_1)(1-t'_2)} + \sqrt{t'_1t'_2}\right]\\
	\implies	1 - t'_1 &= (1-t'_1)(1-t'_2) + t'_1t'_2 + 2\sqrt{(1-t'_1)(1-t'_2)t'_1t'_2}\\
	\implies t'_1 &= \frac{1}{2}\left(1 - \sqrt{1- t'_2}\right) = \frac{1}{2}\left(1 - \sqrt{1- t_1}\right)
	\end{align*}
	Similarly, $r'_{2{M}-3} = r'_{2{M-2}}$ when $t'_{2{M}-3} = \frac{1}{2}\left(1 + \sqrt{t_{M-2}}\right)$. It follows that $r'_1 = r'_2 = r'_{2{M}-3} = r'_{2{M-2}}$ by choosing such $t'_{1},t'_{2M-3}$. 
	
	Next, we find $t'_k \in (t'_{k-1}, t'_{k+1})$ for $k\in\{3, 5, \dots 2{M}-5 \}$ such that the rates $r'_k = r'_{k+1}$. 
	\begin{align*}
	-2 \log \left[\sqrt{(1-t'_k)(1-t'_{k-1})} + \sqrt{t'_kt'_{k-1}}\right] &= -2 \log \left[\sqrt{(1-t'_k)(1-t'_{k+1})} + \sqrt{t'_kt'_{k+1}}\right]\\
	\implies t'_k = \frac{c}{1+c},\text{ where } c&= \left[\frac{\sqrt{1-t'_{k+1}}- \sqrt{1-t'_{k-1}}}{\sqrt{t'_{k-1}}- \sqrt{t'_{k+1}}}\right]^2
	\end{align*}
	Now, we show that $r'_k = r'_j, \forall j,k$ by showing that the difference between each rate $r_i$ and its analogous rate $r'_{2i}$ is constant. $r_k = r_j, \forall j,k$ by assumption and so $r'_k = r'_j, \forall j,k$ follows. 
	
	$r_{M-1} = -\log t_{M-2}$ and $r'_{2{M-2}} =-\log\frac{1}{2}\left(1 + \sqrt{t_{M-2}}\right)$. Thus if $r_i = -\log x$ for some $x$, then $r_{2i} = -\log \frac{1}{2}\left(1 + \sqrt{x}\right)$ would imply that all the rates are equal. Thus, it is sufficient to show that
	\begin{align}
	\left[\sqrt{(1-t'_{2i-1})(1-t'_{2i})} + \sqrt{t'_{2i-1}t'_{2i}}\right]^2 &= \frac{1}{2}\left[1 + \sqrt{(1-t_{i-1})(1-t_{i})} + \sqrt{t_{i-1}t_{i}}\right]\\
	\equiv\left[\sqrt{\left(1-\frac{c}{1+c}\right)(1-t_{i})} + \sqrt{\frac{c}{1+c}t_{i}}\right]^2 &= \frac{1}{2}\left[1 + \sqrt{(1-t_{i-1})(1-t_{i})} + \sqrt{t_{i-1}t_{i}}\right] \label{eqnpart:tedious}\\
	\text{where } c &= \left[\frac{\sqrt{1-t_{i}}- \sqrt{1-t_{i-1}}}{\sqrt{t_{i-1}}- \sqrt{t_{i}}}\right]^2 \nonumber
	\end{align}
	The proof for \eqref{eqnpart:tedious} is algebraically tedious and is shown in Remark~\ref{rem:tediousalgebra} below. 
	
	Then, by the proof of Lemma~\ref{lem:systemofequations}, which shows that equalizing the rates inside the minimization terms implies an optimal $\{t_i\}$, $\beta_{2{M-1}}$ has the levels $\{t_i'\}_{i=0}^{2{M-2}}$. 
	
\end{proof}

\begin{remark}
	\begin{align*}
	\left[\sqrt{\left(1-\frac{c}{1+c}\right)(1-t_{i})} + \sqrt{\frac{c}{1+c}t_{i}}\right]^2 &= \frac{1}{2}\left[1 + \sqrt{(1-t_{i-1})(1-t_{i})} + \sqrt{t_{i-1}t_{i}}\right] \\
	\text{where } c &= \left[\frac{\sqrt{1-t_{i}}- \sqrt{1-t_{i-1}}}{\sqrt{t_{i-1}}- \sqrt{t_{i}}}\right]^2 \nonumber
	\end{align*}
	\label{rem:tediousalgebra}
\end{remark}
\begin{proof}
		Let $x = \sqrt{t_i}, y = \sqrt{1 - t_i}, z = \sqrt{t_{i-1}},$ and $w = \sqrt{1 - t_{i-1}}$. Note that $x>z, w>y, y = 1 - x^2, w = 1-z^2$. Then, \[\frac{c}{c+1} = \frac{(y - w)^2}{2 - 2xz - 2yw}\text{, and } \frac{1}{c+1} = \frac{(x - z)^2}{2 - 2xz - 2yw}\]
	(To show the above two equalities, factor out $\frac{1}{(x-z)^2}$ from numerator and denominator, and substitute $y = 1 - x^2, w = 1-z^2$). 
	
	Now, the left hand side:
	\begin{align*}
	&\left[\sqrt{\left(1-\frac{c}{1+c}\right)(1-t_{i})} + \sqrt{\frac{c}{1+c}t_{i}}\right]^2\\
	=& \frac{1}{2 - 2xz - 2yw}\left[\sqrt{(x - z)^2y^2}+ \sqrt{(y - w)^2x^2}\right]^2\\
	=& \frac{(x - z)^2y^2+ (y - w)^2x^2 + 2xy(x - z)(w - y)}{2 - 2xz - 2yw} & \sqrt{(y - w)^2} = w-y, \sqrt{(x - z)^2} = x-z\\
	=& \frac{z^2y^2 + w^2x^2- 2wxyz}{2 - 2xz - 2yw}
	\end{align*}
	The right hand side:
	\begin{align*}
	&\frac{1}{2}\left[1 + \sqrt{(1-t_{i-1})(1-t_{i})} + \sqrt{t_{i-1}t_{i}}\right]\\
	&= \frac{1}{2} \left[1 + (wy + xz)\right]
	\end{align*}
	Multiplying both sides by $2 - 2xz - 2yw$, we have:
	\begin{align*}
	\left[\sqrt{\left(1-\frac{c}{1+c}\right)(1-t_{i})} + \sqrt{\frac{c}{1+c}t_{i}}\right]^2 &= \frac{1}{2}\left[1 + \sqrt{(1-t_{i-1})(1-t_{i})} + \sqrt{t_{i-1}t_{i}}\right]\\
	\equiv z^2y^2 + w^2x^2 - 2wxyz&= 1 - (wy + xz)^2\\
	\equiv z^2(1-x^2) + (1-z^2)x^2 - 2wxyz &= 1 - w^2y^2 - x^2z^2 - 2wxyz\\
	\equiv z^2 - 2x^2z^2 + x^2 &= 1 - (1 - z^2)(1 - x^2) - x^2z^2\\
	\equiv 0 &= 0
	\end{align*}
\end{proof}
\begin{corollary}
	Suppose $g$ uniform, i.e. $g_i = 1, \forall i$. $\forall \epsilon>0, \exists M$ s.t. $\forall M'\geq M$, $ r^{M'} < \epsilon$.\label{lem:boundratebyeventually}
\end{corollary}
\begin{proof}	
	Let $M = 2^N, M' = 2^{N+1}-1$, for some $N$. We show that $r^{{M'}} \leq \frac{1}{2} r^{M}$. The corollary follows by noting that $r^{K'} < r^K\, \forall K'>K$ and that $r^K < \infty, \forall K$. 
	
	\begin{align*}
	r^M - r^{M'} &= -\log t^M_{M-2} + \log t^{M'}_{M'-2} \\
	&= -\log t^M_{M-2} + \log \left[\frac{1}{2} + \frac{1}{2} \sqrt{t^{M}_{M-2}}\right] & \text{Lemma }\ref{lem:beta2mfromm}\\	
	&= \log \left[\frac{1}{2}\frac{1}{t^M_{M-2}} + \frac{1}{2}\frac{1}{\sqrt{t^{M}_{M-2}}} \right] \\
	&\geq  -\frac{1}{2}\log t^{M}_{M-2}  & \sqrt{t^{M}_{M-2}} \geq t^{M}_{M-2}\\
	\implies r^{{M'}} &\leq \frac{1}{2} r^{M}
	\end{align*}
\end{proof}

\begin{corollary}
	Suppose $g$ uniform, i.e. $g_i = 1, \forall i$. $\forall \delta>0, \exists N$ s.t. $\forall M\geq N$, $\max_k{t^M_k - t^M_{k-1} } < \delta$. \label{lem:bounddistanceeventually}
\end{corollary}
\begin{proof}
	This corollary follows directly from Corollary~\ref{lem:boundratebyeventually}. If the rates are upper bounded, then so are the level differences.
	
	We first find where the rate is minimized given a width between levels of $\delta$
	\begin{align*}
	x_{m} &= \arg\min_x -2\log\left[\sqrt{(1 - x - \delta)(1 - x)} + \sqrt{x(x+\delta)}\right]\\
	&= \frac{1}{2} - \frac{1}{2}\delta
	\end{align*}
	Then given an upper bound of $\epsilon$ on the rate, there is a bound on $\delta$ determined by the largest possible difference at levels symmetric around $\frac{1}{2}$. 
	
	\begin{align*}
	r^L &= -2\log\left[2\sqrt{(\frac{1}{2} - \delta)(\frac{1}{2}+\delta)}\right] \\
	&= -\log\left[1 - 4\delta^2\right]\\
	&\geq \epsilon \text{ when } \delta > \frac{1}{2}\sqrt{1 - e^{-\epsilon}}
	\end{align*}
\end{proof}

\begin{lemma}
	Suppose $g$ is non-decreasing in $\theta$. Then, $t_{M-2} \geq 1-\frac{1}{M-1}$. \label{lem:last1m}
\end{lemma}
\begin{proof}
	Note that, with uniform matching, $\forall x \in (0, 1], y \in [0,1-x]$ the rate with values $ t_{i-1} = y, t_{i} = y+x$ is no more than the last with $t_{M-2} = 1-x$. With width $x$, in other words, the extreme points have a larger rates than the middle points. For $i\notin \{1, M-1\}$:
	\begin{align}
	r_i &= \inf_a \{g_{i-1}\text{KL}(a || t_{i-1}) + g_i\text{KL}(a || t_{i})\}\nonumber\\
	&=\inf_a \{\text{KL}(a || y) + \text{KL}(a || y+x)\} & \text{uniform matching}\nonumber\\
	&= -2\log \left[(1-y)^{\frac{1}{2}}(1-y-x)^{\frac{1}{2}} + y^{\frac{1}{2}}(y + x)^{\frac{1}{2}}\right]  \label{eqnpart:pluginr}	\\
	&= -\log \left[(1-y)(1-y-x)+ y(y + x) + 2\left[(1-y)(1-y-x)y(y + x)\right]^{1/2}\right]\nonumber 	\\
	&\leq - \log(1-x) \nonumber
	\end{align} 
	Where line~\eqref{eqnpart:pluginr} follows from line~\eqref{eqnpart:finalsolutionforminimization}. 
	
	By the proof of Lemma~\ref{lem:systemofequations}, the optimal levels equalize the rates between each level. Then, when $g$ is non-decreasing, $g_{M-2} \geq g_{\ell}, \forall \ell\in \{1 \dots M-3\}$. Then, at the same level differences, the rate corresponding to the last level is no smaller. Thus, to equalize the rates, the last width must be no larger than any other width. Thus, $t_{M-2} \geq 1-\frac{1}{L}$. 
\end{proof}

\begin{lemma}
	With uniform matching ($g_i =1$), $r^{2^{N+1}-1} \geq \frac{1}{5} r^{2^{N}}$.\label{lem:lowerboundrate}
\end{lemma}
\begin{proof}
	Let $K = 2^N, K' = 2^{N+1}-1$. Note that $t^K_{K-1} \geq \frac{1}{2}$ by Lemma~\ref{lem:last1m}.  
	\begin{align*}
	r^K - r^{K'} &= -\log t^K_{K-2} + \log t^{K'}_{K'-2} \\
	&= -\log t^K_{K-2} + \log \left[\frac{1}{2} + \frac{1}{2} \sqrt{t^{K}_{K-2}}\right] & \text{Lemma }\ref{lem:beta2mfromm}\\	
	&= \log \left[\frac{1}{2}\frac{1}{t^K_{K-2}} + \frac{1}{2}\frac{1}{\sqrt{t^{K}_{K-2}}} \right] \\
	&\leq  \log \left(t^{K}_{K-2}\right)^{-\frac{4}{5}}  & \frac{1}{2}\left({t^K_{K-2}}\right)^{-1} + \frac{1}{2}\left({{t^{K}_{K-2}}}\right)^{-\frac{1}{2}} \leq \left(t^{K}_{K-2}\right)^{-\frac{4}{5}} \text{ when }\,\, t^{K}_{K-2} \in \left[\frac{1}{2}, 1\right]\\
	\implies r^{{K'}} &\geq \frac{1}{5} r^{K}
	\end{align*}
\end{proof}

\begin{lemma}
	With uniform matching ($g_i =1$), $\exists C>0$ s.t. $\forall M, t^M_1 \geq C M^{-3 } $.
\end{lemma}
\begin{proof}
	By Lemma~\ref{lem:lowerboundrate}, $\exists C_2>0$ s.t. $r^M \geq C_2 5^{-\ceil{\log_2 M} }$. Then
	\begin{align*}
	-\log(1 - t^M_1) &=r^M\\
	&\geq C_2 5^{-\ceil{\log_2 M} }\\
	\implies t^M_1 &\geq 1 - \exp\left[-C_2 5^{-\ceil{\log_2 M} }\right] \\
	&\geq 1 - \exp\left[-C_3                                      M^{-\frac{1}{\log_5 2} }\right] \\
	&\geq \frac{e-1}{e} C_3 M^{-\frac{1}{\log_5 2} } & e^{-x} \leq 1 - \frac{e-1}{e}x \text{ for }  x\in [0,1]\\
	\implies \exists C>0 \text{ s.t. } t^M_1 &\geq C M^{-3 }
	\end{align*}
\end{proof}

\begin{corollary}
	With monotonically non-decreasing $g$, $\exists C>0$ s.t. $\forall M, t^M_1 \geq C M^{-3 } $. \label{cor:t1lowerbound}
\end{corollary}
\begin{proof}
	The result follows from noting that $t^M_1$ with uniform matching lower bounds the first value with any other monotonically non-decreasing $g$, which is a direct application of Lemma~\ref{lem:shiftgmass} -- scale $g$ such that $g_1 = 1$. Then, $g_j\geq 1, j>1$ and $g_0 \leq 1$. Then, the condition of the lemma holds.
\end{proof}

\begin{lemma}
	The run-time of \textit{NestedBisection} is $O(M \log^2 \frac{1}{\delta})$, where $\delta$ is the bisection grid width and $M$ is the number of intervals. 
\end{lemma}
\begin{proof}
	The outer bisection, in \textit{main}, runs at most $\log_2\frac{2}{\delta} + 1$ iterations. Each outer iteration calls \textit{BisectNextLevel} $M-3$ times, and the inner bisection in each call runs for at most $\log_2\frac{2}{\delta}$ iterations. Thus the run-time of algorithm is $ O(M \log^2 \frac{1}{\delta})$. 
\end{proof}

\subsection{Proof for Theorem~\ref{thm:algfinds}}
\label{sec:provealgorithm}
\textit{Finally, we are ready to prove Theorem~\ref{thm:algfinds}. It follows from formalizing the relationship between $\delta$, the bisection grid width, and $\epsilon$, the additive approximation error in the rate function. }

\begin{proof}
	Recall $M$ is the number of intervals (levels) in $\beta$. We use $j, t, t^*$ to denote the levels in a certain iteration, the returned levels, and the optimal levels, respectively. We use $r(\cdot)$ to denote the individual rates between returned levels, i.e. $r(1) = -g_{1}\log (1-t_1)$, $r(m) = \{g_{m-1}\text{KL}(a_m || t_{m-1}) + g_{m}\text{KL}(a_m || t_{m})\}, m\in \{2 \dots M-2\}, r(M-1) = -g_{M-2}\log (t_{M-2})$, and use $r^*$ to denote the optimal rate.  
	
	By Lemma~\ref{lem:last1m}, $t^*_{M-2} \geq 1-\frac{1}{M-1}$. By assumption, $t^*_{M-2} < 1-\delta $. Thus, $t^*_{M-2}\in[1-\frac{1}{M-1}, 1-\delta]$, the starting interval for the outer bisection. 
	
	First, suppose the outer bisection terminates such that $t_{M-2} \leq t^*_{M-2}+\delta$. We prove that this case always occurs below.

	In this case, $r^* - r(M-1)$ is at most $-g_{M-2} \log (t^*_{M-2}) + g_{M-2} \log (t^*_{M-2} + \delta) = g_{M-2} \log \left(\frac{t^*_{M-2} + \delta}{t^*_{M-2}}\right)$. For all $m\in \{M-2 \dots 2\}$, in the final \textit{CalculateOtherLevels} call the algorithm will use bisection to match the corresponding rate with this last rate, $r(M-1) = -g_{M-2} \log (t_{M-2})$, setting $t_{m-2}$ to the smallest value such that $r(m) \leq r(M-1)$ (i.e. the right end of the final interval is chosen). 
	
	Then, $\forall m\in \{M-2 \dots 2\}$, $r(m) \in [r(M-1) - \epsilon(\delta), r(M-1)]$, where $\epsilon(\delta)$ is an upper bound on the change in the rate functions with a shift of $\delta$ in one of the parameters. 
	
	For now, assume $r(1) = -g_{1}\log (t_{1})\geq r(M-1)$. We prove that this occurs below. Then,
	\begin{align*}r(m) &\geq r(M-1) - \epsilon(\delta) & \forall m \in \{1 \dots M\}\\&\geq -g_{M-2} \log \left(t^*_{M-2} + \delta\right) - \epsilon(\delta)
	\end{align*}

	Now we characterize $\epsilon(\delta)$ in the region $[t^*_1 + \delta, t^*_{M-2}+\delta]$. In particular, we want to bound the rate loss from the other levels $r(m), m>1$ after the $g_{M-2}\log\left(\frac{t^*_{M-2}+ \delta}{t^*_{M-2}}\right)$ loss in in $r(M-1)$. Note that the only source of error is a level shifting right by $\delta$. $r_j(\cdot)$ denotes individual rates between levels $j$ in an intermediary iteration. 
	Let $a_i'$ be the minimum point inside the rate infimum after the shift by $\delta$.
	\begin{align*}
	\epsilon(\delta) &= \sup_{t_{i-1}, t_i} \left[g_{i-1} \text{KL}(a_i || t_{i-1}) + g_i \text{KL}(a_i || t_{i}) - g_{i-1} \text{KL}(a_i' || t_{i-1} + \delta) - g_i \text{KL}(a_i' || t_{i})\right]\\
	&\leq \sup_{t_{i-1}, t_i} \left[g_{i-1} \text{KL}(a_i' || t_{i-1}) + g_i \text{KL}(a_i' || t_{i}) - g_{i-1} \text{KL}(a_i' || t_{i-1} + \delta) + g_i \text{KL}(a_i' || t_{i})\right] & \text{$a_i$ is $\inf$ point}\\
	&= \sup_{t_{i-1}, t_i} g_{i-1}\left[a_i'\log\frac{t_{i-1} + \delta}{t_{i-1}} + (1-a_i')\log\frac{1 - t_{i-1} - \delta}{1-t_{i-1}}\right]\\
	&\leq \sup_{t_{i-1}, t_i} g_{i-1}\left[a_i'\log\frac{t_{i-1} + \delta}{t_{i-1}}\right] & \text{2nd term negative}\\
	&\leq g_{M-2}\left[\log\frac{t^*_{1} + \delta}{t^*_{1}}\right] & t_{j} \geq t^*_{1}, g_j \leq g_{M-2}\\
	\implies r(m) &\geq r^* - g_{M-2} \log\left(\frac{t^*_{M-2}+ \delta}{t^*_{M-2}}\right) - g_{M-2} \left[\log\frac{t^*_{1} + \delta}{t^*_{1}}\right]\\
	&\geq r^* - g_{M-2} \frac{\delta}{t^*_{M-2}}- g_{M-2} \frac{\delta}{t^*_{1}} & \log (1 + x) \leq x\\
	&\geq r^* - \delta g_{M-2} \left[ \frac{M-1}{M-2} + \frac{1}{t^*_{1}}\right] & t^*_{M-2} \geq 1-\frac{1}{M-1}\\
	\end{align*}
	By Corollary~\ref{cor:t1lowerbound}, $\exists C>0$ s.t. $t_1^* \geq CM^{-3} \implies r(m) \geq r^* - \delta g_{M-2} \left[ \frac{M-1}{M-2} + CM^3\right]$. Then, let $\delta = \frac{\epsilon}{g_{M-2} \left[ \frac{M-1}{M-2} + CM^3\right]}$. Supposing the algorithm terminates in such an iteration, it finds an $\epsilon$-optimal $\beta$ in time $O\left(M \log^2 \frac{g_{M-2} \left[ \frac{M-1}{M-2} + CM^3\right]}{\epsilon} \right) = O\left(M \log^2 \frac{M}{\epsilon} \right)$. 
	\\\\Next, we show that the algorithm only terminates the outer bisection when $u \leq t^*_{M-2}+\delta$. The claim follows from $\ell \leq t^*_{M-2}$ being an algorithm invariant. The initial $\ell = 1 - \frac{1}{M-1} \leq t^*_{M-2}$ by Lemma~\ref{lem:last1m}. $\ell$ can only be set to be $> t^*_{M-2}$ if in the current iteration, $j_{M-2} > t^*_{M-2}$ and $r_j(1) < r_j(M-1)$. However, if $j_{M-2} \geq t^*_{M-2}$, then $r_j(1) \geq r_j(M-1)$ ($j_m \geq t^*_m \forall m$), following from a shifting argument like that given in Lemma~\ref{lem:systemofequations} and that the inner bisection is such that $r_j(m) \leq r_j(M-1), m \in \{2 \dots M-2\}$, i.e. all the values $t_m > t^*_m$. Thus, $\ell \leq t^*_{M-2}$ is an algorithm invariant and $u > t^*_{M-2}+\delta \implies u - \ell > \delta$. 
	\\\\Finally, we show that $r(1) \geq r(M-1)$ at the returned $\{t_i\}$. By assumption, in the initial iteration, $u \geq t^*_{M-2}$, and recall that the returned $\{t_i\}$ such that $t_{M-2} = u$ from the final iteration. As shown in the previous paragraph, $j_{M-2} \geq t^*_{M-2} \implies r_j(1) \geq r_j(M-1)$. Thus, if the algorithm terminates in the first iteration, then  $r(1) \geq r(M-1)$. In any subsequent iteration, $u$ is changed only if $r_j(1) \geq r_j(M-1)$ at its new value. Thus, $r_j(1) \geq r_j(M-1)$ is an algorithm invariant, and $r(1) \geq r(M-1)$. 
	\\\\The algorithm terminates in finite time. Thus, it terminates when $t_{M-2} = u \leq t^*_{M-2}+\delta$ and finds a $(\epsilon, M, g)$-optimal $\beta$ in time $O\left(M \log^2 \frac{M}{\epsilon} \right)$. 
	
\end{proof}

In Theorem~\ref{thm:algfinds}, there is an guarantee of an additive error away from the optimal rate. To instead have a multiplicative error bound for uniform matching, one can use the lower bound on the optimal rate from Lemma~\ref{lem:lowerboundrate}, $\exists C>0$ s.t. $r^* \geq C M^{-3}$. Then, for uniform matching, the algorithm returns a $(1 - \epsilon)$ multiplicative approximation in time $O\left(M \log^2 \frac{M}{\epsilon} \right)$. 

\subsection{Proof of Theorem~\ref{thm:unifconvergencecleanmain}}
\label{sec:proveconvergence}

Let $\beta^w_M$ denote the optimal $\beta$ with $M$ intervals for weight function $w$, with intervals $\v{s^{wM}}$ and levels $\v{t^{wM}}$. Let $q_{wM}(\theta) = i/M$ when $\theta\in[s^{wM}_i, s^{wM}_{i+1})$, i.e. the \textit{quantile} of interval item of type $\theta$ is in. Then we have the following convergence result for $\beta_M$.

\thmunifconvergencecleanmain*

\begin{proof} 
	
	Note that the condition on $q$ implies that $\exists \b{M}$ s.t. $\forall M>\b{M},\forall \theta, \exists x_\theta$ such that $\theta \in \Big[s^{M}_{\floor{x_\theta M}}, s^{M}_{\ceil{x_\theta M}}\Big)$. 
	
	Let $ M' = 2M-1, M'' = 4M-3, M^q = 2^{q}M - 2^q + 1$. $\theta \in \Big[s^{M}_{\floor{x_\theta M}}, s^{M}_{\ceil{x_\theta M}}\Big) \implies \beta_M(\theta) = t^M_{\floor{x_\theta M}} \in \left[t^M_{\floor{x_\theta M}-1}, t^M_{\floor{x_\theta M}+1}\right]$. Then, \begin{align*}
	\beta_{M'}(\theta) &= t^{M'}_{\floor{x_\theta M'}}\\
	&= t^{M'}_{\floor{x_\theta(2M-1)}}\\ 
	&\in \left[t^{M'}_{2\floor{x_\theta M}- 2}, t^{M'}_{2\floor{x_\theta M}+ 2}\right]\\
	&\subset \left[t^{M}_{\floor{x_\theta M}-1}, t^{M}_{\floor{x_\theta M}+ 1}\right] & \text{Lemma }\ref{lem:beta2mfromm}\\
	\end{align*}
	And, for general $q$,
	\begin{align*}
	\beta_{M^q}(\theta)&= t^{M^q}_{\floor{x_\theta(2^qM  - 2^q + 1)}} \\&\in\left[t^{M^q}_{\floor{x_\theta 2^q M} - 2^q}, t^{M^q}_{\floor{x_\theta(2^qM)}  + 1} \right]	\\
	&\subset \left[t^{M^q}_{2^q\floor{x_\theta M} - 2^q}, t^{M^q}_{2^q\floor{x_\theta M} +1}\right]\\
	&\subset \left[t^{M}_{\floor{x_\theta M} - 1}, t^{M}_{\floor{x_\theta M} +1}\right]&\text{Lemma }\ref{lem:beta2mfromm}
	\end{align*}
	Then, $\forall N' > 1, \theta$: $\beta_{2^{N'}M - 2^{N'} + 1}(\theta) \in   \left[t^{M}_{\floor{x_\theta M}-1}, t^{M}_{\floor{x_\theta M}+ 1}\right]$ and 
	$$|\beta_{2^{N'}M - 2^{N'} + 1}(\theta) - \beta_{M}(\theta)| \leq t^{M}_{\floor{x_\theta M}+ 1} - t^{M}_{\floor{x_\theta M}-1}$$
	
	By Corollary~\ref{lem:bounddistanceeventually}, $\forall \delta >0, \exists K$ s.t. $\forall K' > K$, $t^{K'}_{\floor{x_\theta K'}+ 1} - t^{K'}_{\floor{x_\theta K'}-1} < 2\delta$.
	
		By the Cauchy criterion, $\exists \beta$ s.t. $\beta_{(C-1)2^N+1}\to \beta $ uniformly. 
	
	By change of variables, $\exists \beta$ s.t. $\beta_{C2^N+1}\to \beta $ uniformly. 
\end{proof}

\begin{corollary}
	For Kendall's tau and Spearman's rho correlation measures, $\exists \beta$ s.t. $\beta_{2^N} \to \beta$ uniformly as $N\to\infty$. \label{cor:kendallspearmanconverges}
\end{corollary}
\begin{proof}
For Kendall's tau and Spearman's rho, $\{s_i\}$ is spaced such that $\forall i,j, s_{i} - s_{i-1} = s_{j} - s_{j-1}$. Thus, $x_\theta = \theta$ meets the criterion. 
\end{proof}

\subsection{Kendall's tau and Spearman's rho related proofs}
\label{sec:appktsr}
\begin{definition} [see e.g. \citet{nelsen_introduction_2007,embrechts_chapter_2003}]
	The population version of Kendall-tau correlation between item true quality and rating scores is proportional to
	\begin{equation*}
	W_k^{\tau} \triangleq 2\int_{\theta_1>\theta_2} P_k(\theta_1, \theta_2) d\theta_1d\theta_2
	\end{equation*}
	Similarly, given items with qualities $\theta_1, \theta_2, \theta_3$, the population version of Spearman's rho correlation between item true quality and rating scores is
	\begin{equation*}
	W_k^{\rho} \triangleq 6\int_{\theta_1 > \theta_2, \theta_3} P_k(\theta_1, \theta_3) d\theta_1 d\theta_2d\theta_3
	\end{equation*}
\end{definition}

\begin{lemma}
	Spearman's $\rho$ can also be written as being proportional to $\int_{\theta_1> \theta_2} (\theta_1 - \theta_2) P_k(\theta_1, \theta_2)d\theta_1d\theta_2$, i.e. with $w(\theta_1, \theta_2) = (\theta_1 - \theta_2)$. 

\end{lemma}
\begin{proof}
	Recall $P_k(\theta_1, \theta_3) = $
	\\\noindent$Pr((\theta_1 - \theta_2)(x^k_1 - x^k_3) > 0)$
	\begin{align*}
	&= \int_{\theta_1 > \theta_2, \theta_3} Pr(x^k_1 - x^k_3 > 0) d\theta_1d\theta_2d\theta_3 + \int_{\theta_1 < \theta_2, \theta_3} Pr(x^k_1 - x^k_3 < 0) d\theta_1d\theta_2d\theta_3\\
	&= \int_{\theta_1, \theta_3} Pr(x^k_1 - x^k_3 > 0) \left[\int_{\theta_2 = 0}^{\theta_1}  d\theta_2\right]d\theta_1d\theta_3  + \int_{\theta_1, \theta_3} Pr(x^k_1 - x^k_3 < 0) \left[\int_{\theta_2 = \theta_1}^{1}  d\theta_2\right]d\theta_1d\theta_3\\
	&= \int_{\theta_1, \theta_3} \left[Pr(x^k_1 - x^k_3 > 0) \theta_1] + Pr(x^k_1 - x^k_3 < 0) (1 - \theta_1)  \right]d\theta_1d\theta_3\\
	&= \int_{\theta_1, \theta_3} \left[Pr(x^k_1 - x^k_3 < 0) + \theta_1 \left[Pr(x^k_1 - x^k_3 > 0) -  Pr(x^k_1 - x^k_3 < 0)\right]\right]d\theta_1d\theta_3
	\end{align*}
	Similarly,
	\\\noindent$Pr((\theta_1 - \theta_2)(x^k_1 - x^k_3) < 0)=$
	\begin{align*}
	&= \int_{\theta_1, \theta_3} \left[Pr(x^k_1 - x^k_3 > 0) + \theta_1 \left[Pr(x^k_1 - x^k_3 < 0) -  Pr(x^k_1 - x^k_3 > 0)\right]\right]d\theta_1d\theta_3\\
	&= \int_{\theta_1, \theta_3} \left[Pr(x^k_3 - x^k_1 > 0) + \theta_3 \left[Pr(x^k_3 - x^k_1 < 0) -  Pr(x^k_3 - x^k_1 > 0)\right]\right]d\theta_1d\theta_3
	\end{align*}
	Where the second equality follows from $\theta_1, \theta_3$ interchangeable. Then
	\\\noindent
	\begin{align*}W^{\rho}_k &= 3\int_{\theta_1, \theta_2} (\theta_1 - \theta_2) P_k(\theta_1, \theta_2)d\theta_1d\theta_2\\
	&= \int_{\theta_1> \theta_2}6 (\theta_1 - \theta_2) P_k(\theta_1, \theta_2)d\theta_1d\theta_2\\ 
	\end{align*}
	
	
\end{proof}

Note that \text{Spearman's $\rho$} is similar to Kendall's $\tau$ with an additional weighting for how far apart the two values that are flipped are. 

\begin{lemma} When $w$ is constant, i.e. for Kendall's $\tau$ rank correlation, the intervals $\v{s}$ that maximize~\eqref{eq:limWk}, \begin{equation}\sum_{0  \leq i < j < M} \int_{\theta_2\in S_i, \theta_1\in S_j } w(\theta_1, \theta_2)d(\theta_1,\theta_2) = \sum_{0  \leq i < j < M}  (s_{i+1} - s_{i}) (s_{j+1} - s_{j})\end{equation} , are $\{s_i = \frac{i}{M}\}_{i=0}^M$.  \label{lem:equidistantvalues} 
\end{lemma}

\begin{lemma} 

	When $w$ is $(\theta_1 - \theta_2)$, i.e. for Spearman's $\rho$ rank correlation, the intervals $\v{s}$ that maximize~\eqref{eq:limWk}, \begin{equation}\sum_{0  \leq i < j < M} \int_{\theta_2\in S_i, \theta_1\in S_j } w(\theta_1, \theta_2)d(\theta_1,\theta_2)\end{equation} are $\{s_i = \frac{i}{M}\}_{i=0}^M$, i.e. the same as those for Kendall's $\tau$.
\end{lemma}
\begin{proof}	
	\begin{align*}
	\sum_{0  \leq i < j < M} \int_{\theta_2\in S_i, \theta_1\in S_j } w(\theta_1, \theta_2)d(\theta_1,\theta_2) &= \sum_{0  \leq i < j < M} \int_{\theta_2\in S_i, \theta_1\in S_j } (\theta_1 - \theta_2) d(\theta_1,\theta_2)\\
	&=\sum_{0< i < j \leq M} \left(\frac{s_j + s_{j-1}}{2} - \frac{s_i + s_{i-1}}{2}\right) (s_i - s_{i-1})(s_j - s_{j-1}) \\
	\end{align*}
	Finding an asymptotically optimal $\{s_i\}$ then is a constrained third order polynomial maximization problem with $M$ variables. The maximum is achieved at $\{s_i = \frac{i}{M}\}_{i=0}^{i=M}$, as for Kendall's tau correlation.
\end{proof}
\end{appendices}

\end{document}